%% file: main.tex
\documentclass[english]{article}
\usepackage[latin9]{inputenc}
\usepackage{array}
\usepackage{float}
\usepackage{booktabs}
\usepackage{multirow}
\usepackage{amsmath}
\usepackage{amsthm}
\usepackage{amssymb}
\usepackage{graphicx}

\makeatletter

\providecommand{\tabularnewline}{\\}
\floatstyle{ruled}
\newfloat{algorithm}{tbp}{loa}
\providecommand{\algorithmname}{Algorithm}
\floatname{algorithm}{\protect\algorithmname}

\theoremstyle{plain}
\newtheorem{thm}{\protect\theoremname}
\theoremstyle{plain}
\newtheorem{prop}{\protect\propositionname}

\usepackage{aaai25}  % DO NOT CHANGE THIS
\usepackage{times}  % DO NOT CHANGE THIS
\usepackage{helvet}  % DO NOT CHANGE THIS
\usepackage{courier}  % DO NOT CHANGE THIS
\usepackage[hyphens]{url}  % DO NOT CHANGE THIS
\usepackage{graphicx} % DO NOT CHANGE THIS
\usepackage{natbib}  % DO NOT CHANGE THIS AND DO NOT ADD ANY OPTIONS TO IT
\usepackage{caption} % DO NOT CHANGE THIS AND DO NOT ADD ANY OPTIONS TO IT
\usepackage{newfloat}
\usepackage{listings}
\DeclareCaptionStyle{ruled}{labelfont=normalfont,labelsep=colon,strut=off} % DO NOT CHANGE THIS
\floatstyle{ruled}
\newfloat{listing}{tb}{lst}{}
\floatname{listing}{Listing}
\title{Learning Structural Causal Models from Ordering: Identifiable Flow Models}
\author{
    Minh Khoa Le \textsuperscript{\rm 1},
    Kien Do \textsuperscript{\rm 1},
    Truyen Tran \textsuperscript{\rm 1}
}
\affiliations{
    \textsuperscript{\rm 1} Applied Artificial Intelligence Institute (A2I2), Deakin University, Australia\\
    \{minh.le, k.do, truyen.tran\}@deakin.edu.au
}

\usepackage{algorithm}
\usepackage{algpseudocode}
\usepackage{algorithmicx}
\usepackage[math]{cellspace}
\@ifundefined{showcaptionsetup}{}{%
 \PassOptionsToPackage{caption=false}{subfig}}
\usepackage{subfig}
\makeatother

\usepackage{babel}
\providecommand{\propositionname}{Proposition}
\providecommand{\theoremname}{Theorem}

\begin{document}
\maketitle 
\begin{abstract}
\input{abstract.tex}

\end{abstract}
\global\long\def\Expect{\mathbb{E}}%
\global\long\def\Real{\mathbb{R}}%
\global\long\def\Data{\mathcal{D}}%
\global\long\def\Loss{\mathcal{L}}%
\global\long\def\Normal{\mathcal{N}}%
\global\long\def\sg{\text{sg}}%
\global\long\def\argmin#1{\underset{#1}{\text{argmin }}}%
\global\long\def\argmax#1{\underset{#1}{\text{argmax }}}%
\global\long\def\network{\text{MAVEN}}%

\section{Introduction}

\input{intro.tex}

\section{Related Work}

\input{related_work.tex}

\section{Preliminaries}

\input{preliminaries.tex}

\section{Method}

\input{method.tex}

\section{Experiments}

\input{experiment.tex}

\section{Conclusion}

\input{discussion.tex}

\bibliography{aaai25}

\newpage{}

\appendix
\onecolumn\input{appendix.tex}

\end{document}

%% file: abstract.tex
In this study, we address causal inference when only observational
data and a valid causal ordering from the causal graph are available.
We introduce a set of flow models that can recover component-wise,
invertible transformation of exogenous variables. Our flow-based methods
offer flexible model design while maintaining causal consistency regardless
of the number of discretization steps. We propose design improvements
that enable simultaneous learning of all causal mechanisms and reduce
abduction and prediction complexity to linear $O(n)$ relative to
the number of layers, independent of the number of causal variables.
Empirically, we demonstrate that our method outperforms previous state-of-the-art
approaches and delivers consistent performance across a wide range
of structural causal models in answering observational, interventional,
and counterfactual questions. Additionally, our method achieves a
significant reduction in computational time compared to existing diffusion-based
techniques, making it practical for large structural causal models.

%% file: intro.tex
Deep neural networks are highly expressive and learnable, but are
inherently associative, making it difficult for them to capture causal
relationships. This limitation can lead to inaccurate predictions
in fields where causality is crucial \cite{leroy2004causality,russo2007interpreting,nguyen2023causal}.
Among the efforts to alleviate this problem, a promising direction
dubbed causal representation learning \cite{scholkopf2021toward}
is to integrate neural networks within the framework of Structural
Causal Models (SCMs) \cite{Pearl2009}. SCMs are principled way to
answer observational, interventional, and counterfactual questions,
but learning them from data remains challenging. This paper focuses
on efficient learning of SCMs from only observational data and causal
ordering, leveraging deep neural networks to model causal relationships
in complex systems.

Some previous methods such as \cite{SanchezMartin2022} require a
fully observed causal graph, which can be infeasible in real-world
settings. Others \cite{javaloy2023causal,khemakhem2021causal} use
Autoregressive Normalizing Flows \cite{papamakarios2017masked,durkan2019neural},
which restrict model design to be monotonic and require additional
regularization to scale to multiple layers. \citet{Sanchez2022DiffusionCM}
propose a counterfactual estimation method using diffusion models
and classifier guidance, but it only considers bivariate causal graphs
and lacks theoretical analysis. \cite{chao2023interventional} generalize
previous diffusion-based methods but require a fully observed causal
graph and a separate deep neural network for each observed variable,
resulting in slow sequential inference and a large number of parameters.

We propose a identifiable flow models that requires only observational
data and a valid causal ordering. Our approach is designed to represent
SCMs and ensure causal consistency through its structure. A parallel
sped-up design can answer observational, interventional, and counterfactual
questions with computational complexity scaling linearly with the
number of model layers, independent of the causal graph's node count.
This scalability enables efficient handling of large causal models,
making our approach practical and effective for diverse applications.

We empirically demonstrate that our method outperforms competing approaches
across a wide range of synthetic and real datasets, excelling in estimating
both the mean and overall shape of interventional and counterfactual
distributions. The experiments confirm that our parallel architecture
is not only scalable but also maintains high performance as the complexity
and size of the datasets increase. 

Our key contributions are:
\begin{enumerate}
\item We prove the identifiability of flow models for learning SCMs from
observational data and causal ordering.
\item We introduce novel model designs enabling parallel abduction and approximated
prediction, removing autoregressive constraints and significantly
reducing computational and memory requirements.
\item We validate our methods' effectiveness and performance on diverse
synthetic and real-world datasets.
\end{enumerate}

%% file: related_work.tex
Recent advances in deep generative models (DGMs) have found their
way into learning Structural Causal Models (SCMs). \citet{Karimi2020}
propose a conditional variational autoencoder (VAE) for each Markov
factorization implied by the causal graph. VAEs on graphs are also
studied assuming certain design constraints but have yet to achieve
empirical success \cite{Zecevic2021,SanchezMartin2022}. \cite{kocaoglu2017causalgan}
use generative adversarial network for learning a causal implicit
generative model for a given causal graph. \cite{geffner2022deep}
propose an autoregressive-flow based non-linear additive noise end-to-end
framework for causal discovery and inference. \cite{Pawlowski2020}
introduce several generative models to learn SCMs, but these are not
guaranteed to learn the true causal mechanism, as multiple models
can produce the same observational distribution. Another DGM class,
autoregressive normalizing flows, has also been suggested \cite{khemakhem2021causal}.
\cite{javaloy2023causal} generalize this by considering a class of
triangular monotonically increasing maps that are identifiable up
to invertible, component-wise transformations. However, this model
design is restricted to monotonic functions and requires additional
regularization to scale to multiple layers. \cite{scetbon2024fip}
view SCMs as a fixed-point problem over causally ordered variables,
infer causal ordering from data and use it to develop a fixed-point
SCM via an attention-based autoencoder.

Diffusion models represent another competitive class of DGMs. \citet{Sanchez2022DiffusionCM}
use diffusion models for counterfactual estimation in bivariate graphs
where an image class is the parent of an image. However, this approach
requires training a separate classifier \cite{dhariwal2021diffusion}
for \emph{do}-interventions, lacks theoretical guarantees, and shows
poor performance for more complex images. \cite{chao2023interventional}
offer both interventional and counterfactual inferences, but only
guarantee identifiability with a fully observed causal graph. Their
method requires a separate neural network for each causal node and
sequential inference, making it computationally expensive and memory-intensive
for large causal graphs. In contrast, our proposed flow models are
flexible in design while maintaining causal consistency. Our approach
reduces computational and memory complexity, enhances scalability
by avoiding separate neural networks for each causal node, and eliminates
sequential inference, making it more practical for large causal graphs.

Alternative to DGMs, Ordinary Differential Equations (ODEs) have been
used to describe deterministic SCMs, which is often unrealistic \cite{mooij2013ordinary}.
\cite{peters2022causal} introduce the Causal Kinetic Model, a collection
of ODEs requiring parent values at each time step. \cite{hansen2014causal}
illustrate a causal interpretation of SDE, show how to apply interventions
to a SDE. \cite{wang2023neural} combine neural SDE with variational
inference to model causal structure in continuous-time time-series
data..

%% file: preliminaries.tex
\subsection{Structural Causal Models (SCMs)}

Given a directed acyclic graph (DAG) graph $\mathcal{G}=\left(\mathcal{V},\mathcal{E}\right)$
representing the causal relationships between $d$ endogenous variables
$x=\left\{ x^{1},...,x^{d}\right\} $, a SCM \cite{Pearl2009} $\mathcal{M}$
associated with $\mathcal{G}$ is a set of structural equations $x^{i}=f^{i}\left(x^{\text{pa}_{i}},u^{i}\right)$
for all $i\in\left\{ 1,...,d\right\} $ that characterize how each
node $x^{i}$ in $\mathcal{V}$ is generated from its parent nodes
$x^{\text{pa}_{i}}:=\left\{ x^{j}\mid\text{the directed edge }\left(j,i\right)\in\mathcal{E}\right\} $
and the corresponding exogenous variable $u^{i}$ via a deterministic
function $f^{i}$. Usually $u=\left\{ u^{1},...,u^{d}\right\} $ are
assumed to be jointly independent, i.e., $p\left(u^{1},...,u^{d}\right)=\prod_{i=1}^{d}p\left(u^{i}\right)$.
This makes the SCM $\mathcal{M}$ Markovian, leading to the factorization
$p\left(x^{1},...,x^{d}\right)=\prod_{i=1}^{d}p\left(x^{i}|x^{\text{pa}_{i}}\right)$.
Since $\mathcal{G}$ is acyclic, we can specify a causal ordering
$\pi$ of all nodes such that if node $x^{j}$ is a parent of node
$x^{i}$ then $\pi\left(j\right)<\pi\left(i\right)$. 

For deep neural networks to answer causal questions, we can treat
$u$ as latent exogenous variables, encode exogenous variables to
latent spaces $z=\tilde{f}_{\text{encode}}\left(x,x^{\text{pa}}\right)$
as the abduction step, and decode back $x=\tilde{f}_{\text{decode}}\left(z,x^{\text{pa}}\right)$
as the prediction step. This structure  enables VAEs, GANs and Normalizing
Flows to learn deep SCMs. 

\paragraph*{Causal consistency}

A mapping $T$ between variables $u$ and $x$ is deemed causally
consistent with structural causal model $\mathcal{M}$ if it shares
the same causal dependencies with the $\mathcal{M}$. It means their
Jacobian matrices have zero values in the same positions, i.e., $\nabla_{u}T\left(u\right)\equiv I+\sum_{i=1}^{\text{diam}(A)}A^{i}$
and $\nabla_{x}T^{-1}\left(x\right)\equiv I-A$, where $I$ is the
identity matrix and $A$ is the adjacency matrix of the causal graph.

\subsection{Triangular Monotonically Increasing (TMI) Maps for Identifiable SCMs }

A function $f:\mathbb{R}^{d}\rightarrow\mathbb{R}^{d}$ is a monotone
increasing triangular map if:
\begin{equation}
f\left(x\right)=\left[\begin{array}{c}
f_{1}\left(x_{1}\right)\\
f_{2}\left(x_{1},x_{2}\right)\\
\vdots\\
f_{d}\left(x_{1},...,x_{d}\right)
\end{array}\right],
\end{equation}
where each $f_{i}:\mathbb{R}\rightarrow\mathbb{R}$ is monotone increasing
(or decreasing) with respect to $x_{i}$ for any $x_{1:i-1}$. In
case $\mu=\nu\cdot f$ where $\mu,\nu$ are strictly positive density
and $f$ is a TMI map, then $f$ is equivalent to the Knothe--Rosenblatt
(KR) transport almost everywhere \cite{jaini2019sum}.

Identifiability refers to recovering ground truth latent factors.
\cite{Xi2023} show that nonlinear independent component analysis
(ICA) models with generator functions that are TMI maps, and fully
supported latent distributions with independent components, are identifiable
up to invertible, component-wise transformations. This means we can
recover the model up to an invertible, component-wise transformation
of the true latent factors. In causal representation learning, on
top of identifying the latent representation, the causal graph encoding
their relations must also be identifiable. \cite{javaloy2023causal}
further show that not only can the model isolate the exogenous variables,
it also shares the functional dependencies with true structural equations.
However, due to the TMI assumption, their model design must fulfill
the monotonic requirements. 

\subsection{Diffusion and Flow Models}

In continuous-time diffusion models, the stochastic process of generating
data from a Gaussian prior can be represented via a backward SDE \cite{song2021scorebased}:
\begin{equation}
dx=\left(f\left(x,t\right)-g^{2}\left(t\right)\nabla_{x}\log p_{t}\left(x\right)\right)dt+g\left(t\right)d\overline{w}
\end{equation}
where $\overline{w}$ denotes the reverse-time Wiener process. The
probability flow ODE (PF ODE) which shares the same marginal probability
densities $p_{t}\left(x_{t}\right)$ as the SDE above is expressed
as follow:
\begin{equation}
dx=\left(f\left(x,t\right)-\frac{1}{2}g\left(t\right)^{2}\nabla_{x}\text{log}p_{t}\left(x\right)\right)dt,
\end{equation}
The PF ODE enables efficient mappings in both directions, from the
data to the prior and vice versa. The PF ODE can be modeled via the
score matching \cite{hyvarinen2005estimation,song2019generative}
or flow matching framework \cite{lipman2023flow,liu2023flow}. In
the latter case, the velocity $v_{\theta}\left(x,t\right)$ of the
PF ODE is learned by minimizing the mean square error w.r.t. the target
velocity $v\left(x_{0},x_{1},t\right)$ over time $t$ sampled uniformly
in $\left[0,1\right]$:
\begin{equation}
\Loss_{\text{FM}}=\mathbb{E}_{t}\Expect_{x_{0},x_{1}}\left[\left\Vert v_{\theta}\left(x_{t},t\right)-v\left(x_{0},x_{1},t\right)\right\Vert ^{2}\right]
\end{equation}
To enable fast sampling, the data and prior distributions can be connected
through a ``stochastic interpolant'' $X_{t}=\left(1-t\right)X_{0}+tX_{1}$
\cite{albergo2023stochastic,liu2023flow}. In this case, $v\left(x_{0},x_{t},t\right)=x_{1}-x_{0}$
and the PF ODE is called the Rectified Flow.

%% file: method.tex
In this section, we first prove that set of flow models are an identifiable
and flexible choice in learning SCMs given only observational data
and causal ordering. For ensuring identifiablity we assume that SCMs
are Markovian, acyclic with diffeomorphic structural equations. Then,
we propose a compact model design that  can do fast inference (prediction
and abduction) in parallel.

\subsection{Identifiable Causal Flow Models for Learning of SCMs with Ordering}

\begin{figure}[t]
\begin{centering}
\includegraphics[width=1\columnwidth]{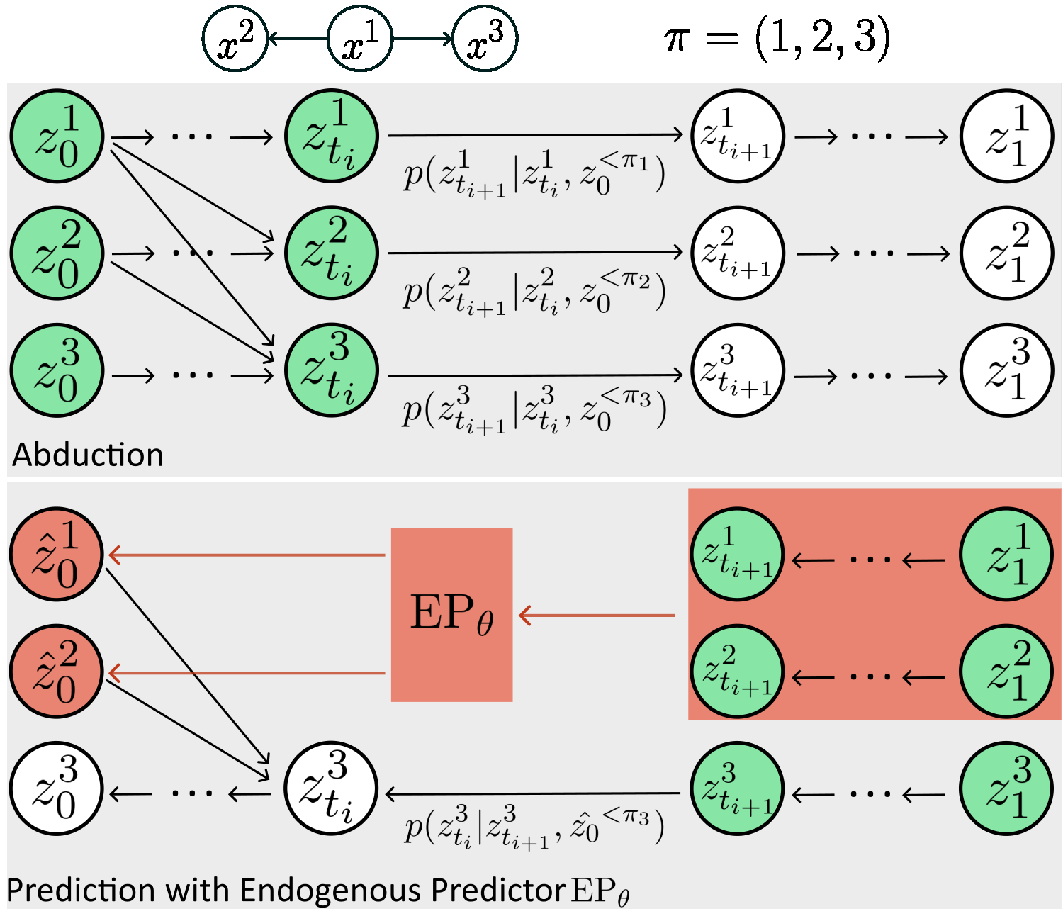}
\par\end{centering}
\caption{Parallel Causal Flow Model (P-CFM) of a simple SCM of 3 nodes (\emph{Top}),
unrolled into abduction (\emph{Middle}) and prediction processes \emph{(Bottom}).
Green nodes are known, white nodes are to be calculated, and orange
nodes are approximated. In abduction, all  $\left\{ p\left(z_{t_{i+1}}^{j}|z_{t_{i}}^{j},z_{0}^{<\pi_{j}}\right)\right\} _{j=1}^{d}$
can be calculated with one forward pass. In prediction, we show an
example of how $p\left(z_{t_{i}}^{3}|z_{t_{i+1}}^{3},\hat{z}_{0}^{<\pi_{j}}\right)$
can be calculated by approximating $z_{0}^{1},z_{0}^{2}$ using the
endogenous predictor $\text{EP}_{\theta}$ given $z_{t_{i+1}}^{1},z_{t_{i+1}}^{2},z_{1}^{1},z_{2}^{1}$.
All $\left\{ p\left(z_{t_{i}}^{j}|z_{t_{i+1}}^{j},\hat{z}_{0}^{<\pi_{j}}\right)\right\} _{j=1}^{d}$
can be calculated simultaneously.\label{fig:Parallel-Causal-Diffusion}}
\end{figure}

We introduce a flow-based model for learning SCMs with ordering that
is identifiable. Our approach involves constructing a set of flows,
one for each node, so that i) the distribution of exogenous variables
is a factorized distribution with support over the entire space $\mathbb{R}^{d}$,
and ii) the mapping from endogenous variables to exogenous variables
is a TMI map. This allows us to leverage a theoretical result from
\cite{Xi2023} (Proposition 5.2) to prove for identifiability.

Let $p\left(u\right)$ represent a jointly factorized distribution
of all $d$ exogenous variables $u=\left\{ u^{1},...,u^{d}\right\} $,
i.e., $p\left(u\right)=\prod_{i=1}^{d}p\left(u^{i}\right)$ where
$p\left(u^{i}\right)$ is a Gaussian distribution in $\Real$. We
assume that while the structural causal model (SCM) $\mathcal{M}$
is unknown, the causal ordering $\pi$ among the nodes is known. We
represent each node $x^{i}\in\Real$ as the value at time $t=1$ of
the initial value problem (IVP) below:
\begin{equation}
dz_{t}^{i}=v^{i}\left(z_{t}^{i},u^{<\pi_{i}},t\right)dt,\;\;\;z_{0}^{i}=u^{i}\label{eq:causal_ivp}
\end{equation}
where $z_{t}^{i}$ denotes the state at time $t$, $u^{<\pi_{i}}:=\left\{ u^{j}\vert\pi_{j}<\pi_{i}\right\} $
is the set of nodes with lower orders than $u^{i}$ according to $\pi$.
The solution at time $t$ of the above IVP can be expressed as $z_{t}^{i}=z_{0}^{i}+\int_{0}^{t}v^{i}\left(z_{\tau}^{i},u^{<\pi_{i}},\tau\right)d\tau$,
which means $z_{t}^{i}$ can be regarded as a function of the initial
value $z_{0}^{i}$. Let $f^{i}\left(u^{i},u^{<\pi_{i}},t\right)=f^{i}\left(z_{0}^{t},u^{<\pi_{i}},t\right):=z_{t}^{i}$
denote the representation of node $i$ at time $t$, and let $f\left(u,t\right):=\left(f^{1}\left(u^{1},u^{<\pi_{1}},t\right),...,f^{d}\left(u^{d},u^{<\pi_{d}},t\right)\right)$
represent the state of all nodes at time $t$. Below, we show that
$f\left(u,t\right)$ is a TMI map of $u$ satisfying two key properties:
\emph{monotonicity} and \emph{triangularity}, as stated in Theorem~\ref{thm:tmi}.
\begin{thm}
Let $f\left(u,t\right)$ be the solution of the set of initial value
problems (IVPs) for all nodes at time $t$, with the IVP for node
$i$ is described in Eq~\ref{eq:causal_ivp}. If the velocity function
$v^{i}\left(z_{t}^{i},u^{<\pi_{i}},t\right)$ is continuous w.r.t.
$t$ and Lipschitz continuous w.r.t. $z_{t}^{i}$ for all $t\in\left(0,1\right)$,
$u^{<\pi_{i}}\in\Real^{\pi_{i}-1}$, and $i\in\left\{ 1,...,d\right\} $,
then $f$ is a triangular monotonically increasing (TMI) map of $u$.\label{thm:tmi}
\end{thm}
\begin{proof}
First, we prove that $f^{i}\left(u^{i},u^{<\pi_{i}},t\right)$ is
a monotonically increasing function of $u^{i}$. Since $u^{<i}$ is
constant w.r.t. $u^{i}$ and $t$, we can simplify the notation by
denoting $\hat{f}^{i}\left(u^{i},t\right):=f^{i}\left(u^{i},u^{<\pi_{i}},t\right)$
and consider $\hat{f}^{i}$ as a function of $u^{i}$ and $t$ only.
Suppose there exists two initial values $a$, $b$ of $u_{i}$ such
that $a>b$ (or equivalently, $\hat{f}^{i}\left(a,0\right)>\hat{f}\left(b,0\right)$)
but $\hat{f}^{i}\left(a,t\right)\leq\hat{f}^{i}\left(b,t\right)$.
Since $\hat{f}^{i}\left(a,t\right)$ and $\hat{f}^{i}\left(b,t\right)$
are continuous for all $t\in\left(0,1\right)$ (due to the continuity
of $v^{i}$ w.r.t. $t$ and $z_{t}^{i}$), there must exists $\gamma\in(0,t]$
such that $\hat{f}^{i}\left(a,\gamma\right)=\hat{f}^{i}\left(b,\gamma\right)=\zeta$
(as illustrated in Fig.~\ref{fig:theorem}). This implies that $\hat{f}^{i}\left(a,t\right)$,
$\hat{f}^{i}\left(b,t\right)$ are two distinct solutions of the IVP:
\begin{equation}
dz_{t}^{i}=v^{i}\left(z_{t}^{i},u^{<\pi_{i}},t\right)dt,\;\;\;z_{\gamma}^{i}=\zeta.\label{eq:contradict_ivp}
\end{equation}
with $t\in\left(\gamma,1\right).$This contradicts the Picard-Lindelof
theorem, which states the IVP in Eq.~\ref{eq:contradict_ivp} has
a unique solution if $v^{i}\left(z_{t}^{i},u^{<\pi_{i}},t\right)$
is continuous w.r.t. $t$ and Lipschitz continuous w.r.t. $z_{t}^{i}$
\cite{simmons2016differential,chen2018neural}. This means $\hat{f}^{i}\left(a,t\right)>\hat{f}^{i}\left(b,t\right)$
if $a>b$, meaning $\hat{f}^{i}\left(u^{i},t\right)$ is a monotonically
increasing function of $u^{i}$. Proving that $f$ is a triangular
map is straightforward since, by design, $f^{i}$ is a function of
nodes with orders lower than or equal to $\pi_{i}$. Consequently,
$f$ is a TMI map of $u$.
\begin{figure}[t]
\begin{centering}
\includegraphics[width=0.8\columnwidth]{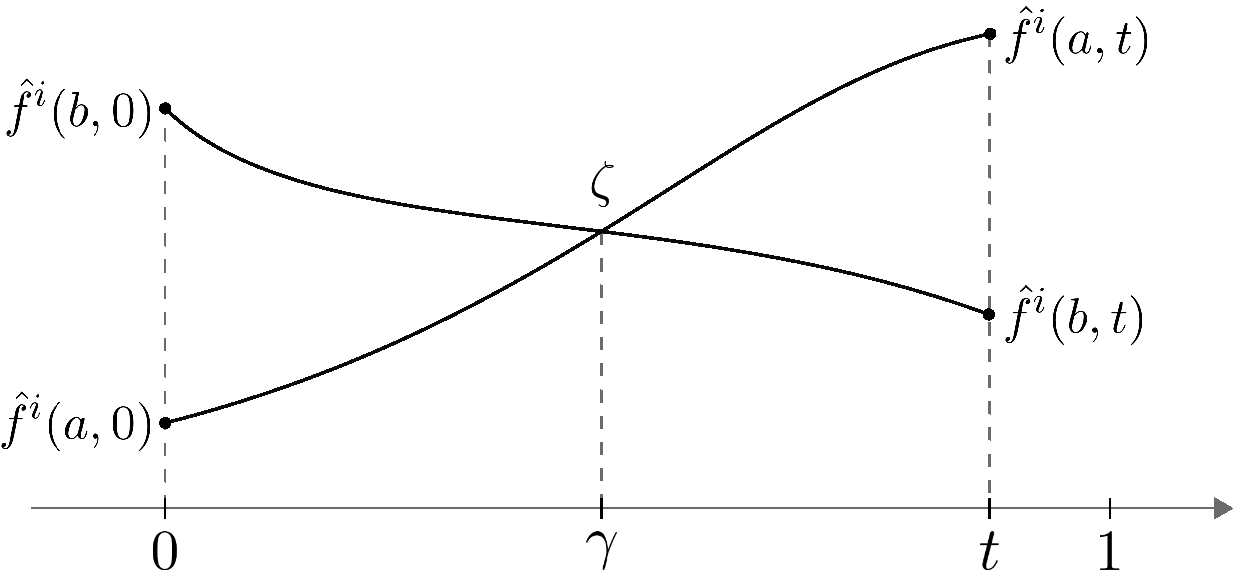}
\par\end{centering}
\caption{An illustration of the case where $\hat{f}^{i}\left(u^{i},t\right)$
is \emph{not} a monotonically increasing function of $u^{i}$ for
all $t\in\left(0,1\right)$. Since $\hat{f}^{i}\left(u^{i},t\right)$
is continuous w.r.t. $t$, we can find $\gamma\in(0,t]$ such that
$\hat{f}^{i}\left(a,\gamma\right)=\hat{f}^{i}\left(b,\gamma\right)=\zeta$.
This leads to two distinct solutions of the IVP starting at $z_{\gamma}^{i}=\zeta$,
which contradicts the Picard-Lindelof theorem.\label{fig:theorem}}
\end{figure}
\end{proof}
Let $v_{\theta}^{i}$ be a parameterized model of $v^{i}$. To ensure
that $v_{\theta}^{i}\left(z_{t}^{i},u^{<\pi_{i}},t\right)$ is continuous
w.r.t. $t$ and Lipschitz continuous w.r.t. $z_{t}^{i}$, we model
it using a feed-forward neural network with finite weights and 1-Lipschitz
continuous activation functions (e.g., sigmoid, tanh, ReLU, ELU).
A detailed explanation for this can be found in Appendix. The input
to this network is the concatenation of $z_{t}^{i}$, $u^{<\pi_{i}},$and
$t$.

Since $u^{<\pi_{i}}$ is typically unknown and often varies during
training $v_{\theta}^{i}\left(z_{t}^{i},u^{<\pi_{i}},t\right)$, we
instead model $v_{\theta}^{i}\left(z_{t}^{i},x^{<\pi_{i}},t\right)$
in practice and consider this velocity in subsequent discussions.
This velocity can be viewed as a reparameterization of $v_{\theta}^{i}\left(z_{t}^{i},u^{<\pi_{i}},t\right)$
because $x^{<\pi_{i}}=f\left(u^{<\pi_{i}},1\right)$.

We note that our method can be applied when the causal graph $\mathcal{G}$
is available. In this case, we simply replace $x^{<\pi_{i}}$ in $v_{\theta}^{i}\left(z_{t}^{i},x^{<\pi_{i}},t\right)$
with $x^{\text{pa}_{i}}$, ensuring that our method remains consistent
with the SCM $\mathcal{M}$.

Since $v_{\theta}^{i}$ represents the velocity of an ODE that maps
$p\left(u^{i}\right)=\Normal\left(0,\mathrm{I}\right)$ to $p\left(x^{i}\right)$,
we learn $v_{\theta}^{i}$ by first constructing a reference diffusion
process between $p\left(x^{i}\right)$ and $p\left(u^{i}\right)$,
and then aligning $v_{\theta}^{i}$ with the velocity of the PF ODE
corresponding to this diffusion process. In this work, we select $Z_{t}^{i}=\left(1-t\right)U^{i}+tX^{i}$
as the reference process due to the sampling efficiency of its PF
ODE. This choice results in the following training objective for all
$v_{\theta}^{i}$ ($i\in\left\{ 1,...,d\right\} $):
\begin{equation}
\mathcal{L}\left(\theta\right)=\sum_{i=1}^{d}\Expect_{t}\Expect_{x^{i},u^{i}}\left\Vert v_{\theta}^{i}\left(z_{t}^{i},x^{<\pi_{i}},t\right)-\left(x^{i}-u^{i}\right)\right\Vert _{2}^{2}+\lambda\left\Vert \theta\right\Vert _{2}^{2}\label{eq:cfm-objective}
\end{equation}
In Eq.~\ref{eq:cfm-objective}, we include a weight regularization
term to ensure that $\theta$ is finite. We refer to our method as
a \emph{Causal Flow Model} (CFM). After training $v_{\theta}^{i}\left(z_{t}^{i},x^{<\pi_{i}},t\right)$
for all $i\in\left\{ 1,...,d\right\} $, we can compute $x^{i}$ from
$u^{i}$ and $x^{<\pi_{i}}$ as $x^{i}=f^{i}\left(u^{i},x^{<\pi_{i}},1\right)$
by solving the ODE in Eq.~\ref{eq:causal_ivp} forward in time. Similarly,
we can compute $u^{i}$ from $x^{i}$ and $x^{<\pi_{i}}$ as $u^{i}=\left(f^{i}\right)^{-1}\left(x^{i},x^{<\pi_{i}},0\right)$
by solving this backward in time. These two steps correspond to the
prediction and abduction steps.

\subsection{Efficient Causal Flow Models}

\subsubsection{Efficient Abduction with Masked Autoregressive Velocity nEural Network
(MAVEN)}

To approximate the velocity $v$ for all nodes while reducing time
and memory complexity, we propose the MAVEN architecture. It comprises
two key components: a Masked Autoregressive Neural Network ($\text{MADE }_{\theta}$)
with a fixed ordering \cite{germain2015made} and a simple feed-forward
neural network $h_{\theta}$. $\text{MADE }_{\theta}$ encodes $x$
into $c\in\mathbb{R}^{d\times k}$, where each vector $c^{i}$ contains
node-specific encoded information dependent only on preceding variables.
The feed-forward network then predicts node velocity by taking a concatenated
vector of $[z_{t}^{i},c^{i},t,i]$.

\begin{align}
c^{i} & =\text{MADE }_{\theta}\left(x^{<\pi_{i}}\right)\\
v_{t}^{i} & =h_{\theta}\left(\left[z_{t}^{i},c^{i},t\right]\right)\\
v_{\theta_{\text{}}}^{\text{\text{}}}\left(z_{t},x,t\right) & =\left[h_{\theta}\left(\left[z_{t}^{1},c^{1},t\right]\right),...,h_{\theta}\left(\left[z_{t}^{d},c^{d},t\right]\right)\right]^{\text{T}}
\end{align}
We enforce the flow of each node to follows a linear path from the
endogenous variable $x^{i}$ to exogenous variable $u^{i}$. This
is achieved through the prior ODE defined as $dx^{i}=\left(u^{i}-x^{i}\right)dt$.
$v_{\theta_{\text{}}}^{\text{\text{}}}$ can be trained using gradient
descent by solving the regression problem:
\begin{equation}
\text{\ensuremath{\theta^{*}=\text{argmin}}}_{\theta}\int_{0}^{1}\mathbb{E}\left[\left\Vert v_{\theta_{\text{}}}^{\text{}}\left(z_{t},x,t\right)-\left(u-x\right)\right\Vert ^{2}\right]dt.
\end{equation}
We call the proposed method as \textbf{Sequential-Causal Flow Model}
(S-CFM). It optimizes the abduction process by reducing computational
complexity from $O(nd)$ to $O(n)$, enabling more efficient processing
of large-scale causal networks. See Abduction in Fig.~\ref{fig:Parallel-Causal-Diffusion}
for an example of 3-node causal graph.

Next, we propose a method that reduce complexity of prediction step
from $O\left(nd\right)$ to $O\left(n\right)$, enabling parallel
prediction.

\subsubsection{Endogenous predictor}

During prediction, we encounter challenges due to limited information
about endogenous variables, necessitating a sequential, node-by-node
prediction approach. At each time step $t$, the $i$-th endogenous
variable $x^{i}$ is deterministically derived from current and preceding
exogenous variables $u^{\leq\pi_{i}}:=\left\{ u^{j}\vert\pi_{j}\leq\pi_{i}\right\} $.
By leveraging a neural network, specifically a MADE, denoted as $\text{EP}_{\theta_{\text{}}}\left(z_{t},u,t\right)$,
we can accurately predict these variables. The predicted endogenous
variables then serve as input for MAVEN to estimate velocity at each
time step, effectively acting as a distillation mechanism for the
flow model.

The endogenous predictor is trained by solving the following optimization
problem:
\begin{equation}
\text{\ensuremath{\theta^{*}=\text{argmin}}}_{\theta}\int_{0}^{1}\mathbb{E}\left[\left\Vert \text{EP}_{\theta}\left(z_{t},u,t\right)-x\right\Vert ^{2}\right]dt.
\end{equation}

We introduce the Parallel-Causal Flow Model (P-CFM) by integrating
S-CFM with an endogenous predictor. The training of the predictor
leverages the abduction process, sampling from $x$ with linear complexity.
This approach ensures efficient training, enabling scalability and
effective handling of larger datasets with minimal computational overhead.

See Fig.~\ref{fig:Parallel-Causal-Diffusion} for an example of P-CFM
on a 3-node SCM. The training of P-CFM is summarized in Alg.~\ref{alg:P-CDM-Model-Training},
the abduction step in Alg.~\ref{alg:abduct} and the prediction step
in Alg.~\ref{alg:predict}.

\begin{table*}[t]
\begin{centering}
\begin{tabular}{>{\raggedright}m{1.8cm}lrrrrr}
\toprule 
\multirow{2}{1.8cm}{\centering{}SCM} & \multirow{2}{*}{Metric} & S-CFM & P-CFM & CausalNF-NSF & CausalNF-MAF & VACA\tabularnewline
\cmidrule{3-7} \cmidrule{4-7} \cmidrule{5-7} \cmidrule{6-7} \cmidrule{7-7} 
 &  & $\left(\times10^{-2}\right)$ & $\left(\times10^{-2}\right)$ & $\left(\times10^{-2}\right)$ & $\left(\times10^{-2}\right)$ & $\left(\times10^{-2}\right)$\tabularnewline
\midrule
\midrule 
\multirow{3}{1.8cm}{\centering{}Triangle\linebreak{}
NADD} & Obs. MMD & $0.67_{0.07}$ & $\textit{0.65}_{\textit{0.15}}$ & $\boldsymbol{0.43_{0.08}}$ & $2.03_{0.14}$ & $2.37_{0.25}$\tabularnewline
 & Int. MMD & $3.21_{0.34}$ & $\textit{3.16}_{0.24}$ & $\boldsymbol{1.86_{0.09}}$ & $13.31_{0.29}$ & $15.53_{1.24}$\tabularnewline
 & CF. MSE & $\boldsymbol{6.53_{0.34}}$ & $\textit{9.59}_{3.53}$ & $41.90_{11.74}$ & $9.91_{1,02}$ & $146.73_{16.29}$\tabularnewline
\midrule 
\multirow{3}{1.8cm}{\centering{}Simpson\linebreak{}
NLIN2} & Obs. MMD & $\boldsymbol{0.57_{0.05}}$ & $\textit{0.61}_{\textit{0.05}}$ & $0.62_{0.05}$ & $3.54_{0.21}$ & $2.44_{0.22}$\tabularnewline
 & Int. MMD & $\boldsymbol{0.58_{0.01}}$ & $\textit{0.60}_{\textit{0.04}}$ & $0.61_{0.01}$ & $4.34_{0.10}$ & $2.55_{0.32}$\tabularnewline
 & CF. MSE & \textbf{$\boldsymbol{0.72_{0.18}}$} & $\textit{0.73}_{\textit{0.18}}$ & $1.98_{0.86}$ & $1.29_{0.20}$ & $15.28_{1.25}$\tabularnewline
\midrule 
\multirow{3}{1.8cm}{\centering{}Diamond\linebreak{}
NLIN} & Obs. MMD & $\boldsymbol{0.36_{0.02}}$ & $\textit{0.38}_{\textit{0.06}}$ & $0.44_{0.04}$ & $2.26_{4.15}$ & $5.23_{0.35}$\tabularnewline
 & Int. MMD & $\textit{0.41}_{\textit{0.05}}$ & $\textit{0.41}_{\textit{0.06}}$ & \textbf{$\boldsymbol{0.34_{0.02}}$} & $4.49_{8.95}$ & $26.64_{0.63}$\tabularnewline
 & CF. MSE & $\textit{0.04}_{\textit{0.04}}$ & $\boldsymbol{0.03_{0.01}}$ & $16.06_{2.07}$ & $197.27_{438.22}$ & $60.62_{3.45}$\tabularnewline
\midrule 
\multirow{3}{1.8cm}{\centering{}Y\linebreak{}
NADD} & Obs. MMD & $\boldsymbol{0.46_{0.05}}$ & $\boldsymbol{0.46_{0.05}}$ & $\textit{0.55}_{\textit{0.07}}$ & $1.01_{0.05}$ & $1.68_{0.08}$\tabularnewline
 & Int. MMD & $\boldsymbol{0.48_{0.03}}$ & $\textit{0.51}_{\textit{0.04}}$ & $1.75_{0.70}$ & $1.56_{0.13}$ & $4.44_{0.20}$\tabularnewline
 & CF. MSE & $\boldsymbol{28.58_{1.04}}$ & $\textit{29.67}_{\textit{0.84}}$ & $37.31_{3.32}$ & $30.27_{0.78}$ & $32.45_{0.72}$\tabularnewline
\midrule
\multirow{3}{1.8cm}{\centering{}LargeBD\linebreak{}
NLIN} & Obs. MMD & $\textit{0.40}_{\textit{0.03}}$ & $\boldsymbol{0.39_{0.02}}$ & $0.42_{0.04}$ & $0.97_{0.06}$ & $1.34_{1.14}$\tabularnewline
 & Int. MMD & $\boldsymbol{0.41_{0.02}}$ & $0.48_{0.03}$ & $\textit{0.43}_{\textit{0.03}}$ & $0.99_{0.03}$ & $0.98_{0.06}$\tabularnewline
 & CF. MSE & $\boldsymbol{0.02_{0.00}}$ & $\textit{0.03}_{\textit{0.00}}$ & $0.04_{0.00}$ & $0.04_{0.00}$ & $0.71_{0.01}$\tabularnewline
\midrule
\multirow{3}{1.8cm}{\centering{}LargeLadder\linebreak{}
NLIN} & Obs. MMD & $\boldsymbol{0.63_{0.01}}$ & $\textit{0.64}_{\textit{0.00}}$ & $0.66_{0.04}$ & $0.65_{0.05}$ & -\tabularnewline
 & Int. MMD & $\boldsymbol{0.64_{0.01}}$ & $\textit{0.65}_{\textit{0.00}}$ & $0.66_{0.00}$ & $\textit{0.65}_{\textit{0.00}}$ & -\tabularnewline
 & CF. MSE & $\textit{42.71}_{\textit{3.12}}$ & $46.79_{6.67}$ & $72.63_{5.4}$ & $\boldsymbol{38.46_{1.30}}$ & -\tabularnewline
\midrule
\multirow{3}{1.8cm}{\centering{}LargeLadder\linebreak{}
NADD} & Obs. MMD & $\boldsymbol{0.58_{0.06}}$ & $\textit{0.61}_{\textit{0.02}}$ & $0.62_{0.03}$ & $0.71_{0.05}$ & -\tabularnewline
 & Int. MMD & $\boldsymbol{0.62_{0.00}}$ & $\textit{0.63}_{\textit{0.00}}$ & $\boldsymbol{0.62_{0.00}}$ & $0.68_{0.00}$ & -\tabularnewline
 & CF. MSE & $\boldsymbol{27.85_{1.39}}$ & $\textit{29.09}_{\textit{0.29}}$ & $36.23_{3.49}$ & $37.42_{1.94}$ & -\tabularnewline
\bottomrule
\end{tabular}
\par\end{centering}
\caption{Mean and standard deviation of five SCMs with non-linear (NLIN) and
non-additive (NADD) structural equations. Bold indicates the best
result; italic indicates the second-best.  \label{tab:synthetic-results}}
\end{table*}

\begin{algorithm}
\caption{P-CFM Model Training.\label{alg:P-CDM-Model-Training}}

\textbf{Input}: Observational data $\mathcal{X}$ , causal ordering
$\pi$, exogenous distribution $p(u)$

\begin{algorithmic}[1]

\While{not converged} \Comment{Training S-CFM}
\State $z_0 = x \sim \mathcal{X}$; ~$z_1 = u \sim p(u)$; ~$t \sim  \text{Unif}[0,1]$; ~$z_t = (1-t)\cdot z_0 + t\cdot z_1$
\State Update parameters of velocity neural network $v_\theta$, by minimizing the following loss: \[ \left\Vert v_{\theta}\left( z_{t},x,t\right) - \left( z_{1} - z_{0} \right) \right\Vert ^{2} \]
\EndWhile
\newline

\While{not converged} \Comment{Training Endogenous Predictor}
\State $t \sim  \text{Unif}[0,1]$; ~$z_0=x \sim \mathcal{X}$
\State Sample $z_t,z_1$ given $z_0$ from S-CFM
\State Update parameters of endogenous predictor neural network $\text{NN}_{\theta}$, by minimizing the following loss: \[ \left\Vert \text{EP}_{\theta}\left(z_{t},z_{1},t\right)-z_{0}\right\Vert ^{2} \] 
\EndWhile

\end{algorithmic}
\end{algorithm}

\paragraph{Do operator}

To enable intervention and counterfactual calculations, we modify
Pearl's \emph{do}-operator $do\left(x^{i}=\alpha\right)$ differently
for S-CFM and P-CFM. In S-CFM, we replace the $i$-th causal mechanism
$f^{i}$ with a constant function, $f_{I}^{i}=\alpha$. For P-CFM,
we adopt the backtracking counterfactual approach \cite{von2023backtracking},
which modifies the exogenous distribution $p\left(u\right)$ while
maintaining causal mechanisms $f$. An intervention $do(x^{i}=\alpha)$
updates $p(u)$ by restricting plausible $u$ values to those yielding
the intervened value $\alpha$. The intervened SCM is defined as $\mathcal{M}^{\mathcal{I}}=\left(f,p_{I}(u)\right)$,
with $p_{I}(u)$ density determined by:

\begin{equation}
p_{I}(u)=\delta\left(\left\{ u^{i}:f_{i}\left(x^{\text{pa}_{i}},u^{i}\right)=\alpha\right\} \right)\cdot\prod_{j\neq i}p\left(u^{j}\right).
\end{equation}
During the prediction step, the endogenous predictor uses intervened
exogenous values to predict endogenous values. Traditional \emph{do}-operators
require finding specific set $\left(u^{i},x^{\text{pa}_{i}}\right)$
that causes the intervened value $\alpha$, necessitating an additional
abduction step that can increase complexity from $O(1)$ to $O(nd)$
during iterative inference. In our approach where the abduction step
can be run in parallel, this operator is computationally cheaper with
a complexity of $O(n)$. As a result, our overall complexity for observational,
interventional and counterfactual queries are $O\left(n\right)$.
More specificially, these queries are detailed as follows:

\begin{itemize}
\item The observation is sampled from $p\left(x\right)$ as $x=\text{Predict}(u)$
for $\ensuremath{u\sim p\left(u\right)}$.
\item The interventional sampling from $p\left(x\mid\text{do}\left(x^{i}=\alpha\right)\right)$
involves the steps: $x=\text{Predict}(u)$ for $\ensuremath{u\sim p\left(u\right)}$,
followed by $x^{i}=\alpha,$ $u^{i}=\text{Abduct}\left(x\right)^{i}$,
and $x^{\text{int}}=\text{Predict}(u)$
\item The counterfactual sampling from $p\left(x^{\text{cf}}|\text{do}\left(x^{i}=\alpha\right),x^{\text{f}}\right)$
involves the steps: $u=\text{Abduct}(x^{\text{f}})$, $x^{i,\text{f}}=\alpha$,
$u^{i}=\text{Abduct}(x^{\text{f}})^{i}$, followed by $x^{\text{cf}}=\text{Predict}(u)$.
\end{itemize}
\begin{algorithm}
\caption{$\text{Predict}(u)$\label{alg:predict}}

\textbf{Input}: $u$, sequence of times $0=t_{1}<t_{2}<...<t_{N-1}<t_{N}=1$

\begin{algorithmic}[1] 

\State $z_1=u$

\For{$n=N \pmb{\text{ to }} 2$}
\State $\hat{z}_0 = \text{EP}_{\theta}\left(z_{1},z_{t_{n}},t_n\right)$ \Comment{Predict endogenous values}
\State $v = v_{\theta}( z_{t_n}, \hat{z}_0, t_n)$
\State $z_{t_{n-1}}=z_{t_n} - v\cdot (t_{n} - t_{n-1})$
\EndFor
\State $\text{Return } z_0$
\end{algorithmic}
\end{algorithm}
\begin{algorithm}
\caption{$\text{Abduct}\left(x\right)$\label{alg:abduct}}

\textbf{Input}: $x$, sequence of times $0=t_{1}<t_{2}<...<t_{N-1}<t_{N}=1$

\begin{algorithmic}[1] 

\State $z_0=x$

\For{$n=1 \pmb{\text{ to }} N-1$}
\State $v = v_{\theta}( z_{t_n}, z_0, t_n)$
\State $z_{t_{n+1}}=z_{t_n} + v\cdot (t_{n+1} - t_{n})$
\EndFor
\State $\text{Return } z_1$
\end{algorithmic}
\end{algorithm}

%% file: experiment.tex
We evaluate the performance of the proposed sequential and parallel
CFM on both synthetic and real-world dataset, as well as a major
speed up by parallelization. 

\subsection{Settings}

\paragraph*{Implementation}

For MAVEN and endogenous predictor $\text{EP}_{\theta}$, we use a
MADE with three hidden layers $\left[256,256,256\right]$ and ELU
activation, and a fully connected neural network with the same layers
and activation. We use the Adam optimizer with a learning rate of
0.001 and apply a decay factor of 0.95 to the learning rate if it
remains at a plateau for more than 10 epochs, batch size of 2048,
and train for 900 epochs. Inference process is done in 50 steps.

\paragraph{Synthetic datasets}

We consider synthetic SCMs various in the number of nodes and edges
as they allow us to have direct access to the true observational,
interventional and counterfactual distributions to evaluate methods.
We study two main classes of structural equations: 

\emph{\textgreater Non-linear additive noise} (NLIN): \textbf{Simpson},
a 4-node SCM simulating a Simpson\textquoteright s paradox that is
difficult to approximate. \textbf{Diamond}, a 4-node SCM with 4 egdes
arranged in a diamond shape. \textbf{Large Backdoor }(LargeBD), a
9-node non-Gaussian noise with sparse causal graph.

\emph{\textgreater Non-additive noise} (NADD): \textbf{Triangle},
a 3-node SCM with confounding graph in which $x^{1}$ is the parent
of $x^{2},x^{3}$ and $x^{2}$ causes $x^{3}$. \textbf{Y}, a 4-node
Y-shaped SCM.

We randomly sample 50,000/5,000/5,000 samples for training, evaluation
and testing, respectively. See appendix for more details on the structural
equations and all results.

\begin{figure*}
\begin{centering}
\includegraphics[scale=0.77]{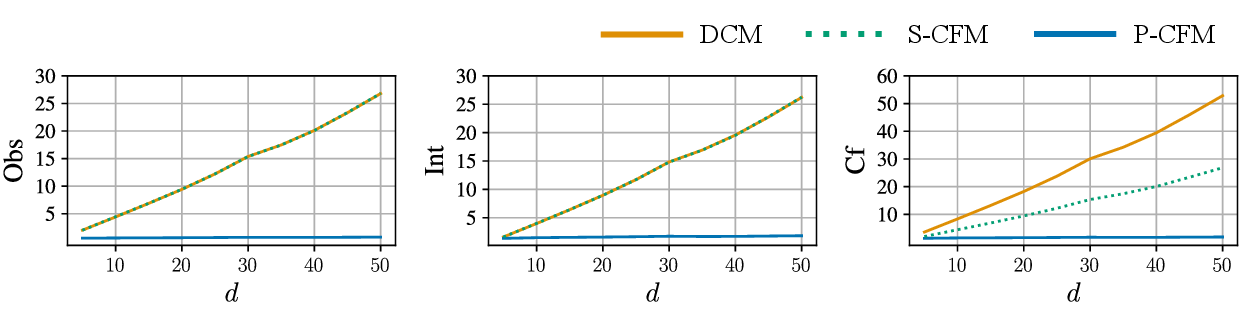}
\par\end{centering}
\caption{Inference time (in seconds) against number of nodes $d$  for 10
samples with 50 discretization steps on \textbf{Obs}ervational, \textbf{Int}erventional,
and \textbf{C}ounter\textbf{f}actual queries, for data generated using
linear structural equations, ranging from $5$ to $50$ nodes.\label{fig:scalability}}
\end{figure*}

\paragraph{Baselines}

For the synthetic experiments, we compare our methods with two recently
state-of-the-art methods Causal Normalizing Flows (CausalNF) \cite{javaloy2023causal}
and VACA \cite{SanchezMartin2022}. For CausalNF, we use two types
of flow architectures: Masked Autoregressive Flow (MAF) \cite{papamakarios2017masked}
and Neural Spline Flow (NSF) \cite{durkan2019neural}. As a requirements
of VACA, we use fully observed causal graph for training in testing
VACA. Due to the differences in settings, we do not show comparison
between our methods and Diffusion-based Causal Models (DCM) \cite{chao2023interventional},
which requires fully observed causal graph. When only the causal
ordering is observed, training the DCM requires either multiple random
selections of parents for each node or full causal graph discovery,
which are both expensive. We re-implement these methods based on
provided code. For a fair comparison, each model uses the same budget
for optimization. For the real datasets, we compare with CausalNF-MAF,
CausalNF-NSF, CAREFL \cite{khemakhem2021causal}, regression function
from an additive noise model (ANM) \cite{hoyer2008nonlinear} and
standard ridge regression model (Linear SCM).

\paragraph{Metrics}

To evaluate the methods on the three-layer causal hierarchy, we report
the Maximum Mean Discrepancy (MMD) \cite{gretton2012kernel} between
the true and the estimated observational and interventional distributions.
We do not evaluate the Average Treatment Effect (ATE) as in \cite{javaloy2023causal},
as the metric does not provide insights into the distribution of treatment
effects within the population. For counterfactual estimation, we report
the mean squared error (MSE) between the actual and estimated counterfactual
values.

\subsection{Results on Synthetic Datasets}

Table~\ref{tab:synthetic-results} reports the experimental results
on five synthetic SCMs for observational, interventional and counterfactual
questions. Each model's performance is averaged over five random initializations
of parameters and training datasets. VACA demonstrates poor performance
across all datasets, which can be attributed to the limitations of
the encoder-decoder architectures and the constraints of the graph
neural network. CausalNF-NSF and CausalNF-MAF show competitive results
compared to ours. However, due to the affine constraints of the Masked
Autoregressive Flow, CausalNF-MAF only learns an affine transformation
between exogenous and endogenous distributions, performing well in
cases with additive structural equations but failing in non-additive
cases where the endogenous distribution is not Gaussian. CausalNF-NSF
closely approximates observational and interventional distributions
in both additive and non-additive cases but struggles with answering
counterfactual questions.

Our methods exhibit a significantly higher and more consistent performance
compared to existing methods across the majority of queries, over
both additive and non-additive relationships. Specifically, P-CFM,
despite approximating endogenous values $x$, demonstrates a superior
performance relative to previous methods. 

\paragraph{Scalability}

Fig.~\ref{fig:scalability} demonstrates the scalability of P-CFM
against previous Diffusion-based methods such as DCM, and our own
S-CFM, all using the same number of parameters. DCM's inference time
grows linearly with the number of nodes $d$. S-CFM significantly
reduces counterfactual generation time through parallel abduction.
In contrast, P-CFM maintains a nearly constant inference time regardless
of $d$. This result highlights that P-CFM is not only accurate but
also efficient, making it suitable for larger problems.

\subsection{Results on Real Dataset}

\begin{table}[h]
\begin{centering}
\begin{tabular}{lc}
\toprule 
Algorithm & Median Abs. Error\tabularnewline
\midrule
\midrule 
S-CFM & $\boldsymbol{0.5829_{5.0\text{e-}2}}$\tabularnewline
CausalNF-NSF & $0.6018_{2.0\text{e-}2}$\tabularnewline
CausalNF-MAF & $0.6025_{3.0\text{e-}2}$\tabularnewline
CAREFL & $\boldsymbol{0.5983_{3.0\text{e-}2}}$\tabularnewline
ANM & $0.6746_{4.8\text{e-}8}$\tabularnewline
Linear SCM & $0.6045_{0.0}$\tabularnewline
\bottomrule
\end{tabular}
\par\end{centering}
\caption{Interventional prediction's median absolute error on fMRI data. The
results for CAREFL, ANM, and Linear SCM are obtained from Chao et
al. (2023) \label{tab:real-results}}
\end{table}

We evaluate S-CFM on real-world setting on the electrical stimulation
interventional fMRI data \cite{thompson2020data}, experimental setup
is obtained from \cite{khemakhem2021causal}. We do not evaluate P-CFM
since it does not show much advantages in case of only two nodes.
The fMRI data includes time series from the Cingulate Gyrus (CG) and
Heschl\textquoteright s Gyrus (HG) for 14 patients with medically
refractory epilepsy. The underlying simplified causal structure is
a bivariate graph with CG $\rightarrow$ HG. Our task is to predict
value of HG given intervened CG.

Table~\ref{tab:real-results} reports the interventional results
on fMRI datasets. S-CFM demonstrates slightly better performance compared
to the alternatives. This is because the evaluation is based on the
absolute error of a single intervention, which does not account for
distributional information. Additionally, the causal ordering is consistent
with the causal structure in the bivariate case. Nonetheless, the
experiment provides a benchmark for comparing various causal inference
algorithms using real datasets.

%% file: discussion.tex
We have introduced a new scalable method of causal learning and inference
using flow models: We proved that a flow-based models can learn identifiable
SCMs from only observational data and causal ordering. To make the
models efficient, we designed a masked autoregressive velocity network
and an endogenous predictor that enables parallel inference. We empirically
demonstrated that our method outperforms SOTA rivals on a wide range
of SCMs over non-additive and non-linear relationships. We showed
that our parallel causal flow model is scalable, maintaining a nearly
constant inference time regardless of the number of variables. We
validated our method on real-world datasets, including an fMRI study,
showcasing its practical applicability.

%% file: appendix.tex
\section{Missing Details from Method section}

\subsection{Identifiability of TMI map}

Below is the proposition for the identifiability of TMI maps from
\cite{Xi2023}, included herein for the purpose of comprehensiveness.
\begin{prop}
Let $\text{Z}=\text{X}=\mathbb{R}^{d}$. The nonlinear ICA model where
$\mathcal{F}$ are TMI maps and $\mathcal{P}_{z}$ are fully supported
distributions with independent components is identifiable up to invertible,
component-wise transformations.
\end{prop}

\subsection{Picard-Lindelof theorem}

The Picard-Lindelof theorem, as discussed by \cite{simmons2016differential}
states the specific conditions under which an initial value problem
guarantees a unique solution. We include it here for the completeness
of the theoretical foundations.
\begin{thm}
Let $f\left(x,y\right)$ be a continuous function that satisfies a
Lipschitz condition
\[
\left|f\left(x,y_{1}\right)-f\left(x,y_{2}\right)\right|\leq K\left|y_{1}-y_{2}\right|
\]
 on a strip defined by $a\leq x\leq b$ and $-\infty<y<\infty$. If
$(x_{0},y_{0})$ is any point of the strip, then the initial value
problem
\[
y'=f\left(x,y\right),\quad y\left(x_{0}\right)=y_{0}
\]
has one and only one solution $y=y\left(x\right)$ on the interval
$a\leq x\leq b$.
\end{thm}

\subsection{Lipschitz continuity of the velocity network\label{subsec:Lipschitz-continuity-of-the-velocity-network}}

Recalling that the velocity for node $i$, $v_{\theta}\left(z_{t}^{i},x^{<\pi_{i}},t\right)$,
is modeled as follows:
\begin{align*}
c^{<\pi_{i}} & =\text{MADE}\left(x^{<\pi_{i}}\right)\\
v_{t}^{i} & =\text{MLP\ensuremath{\left(\left[z_{t}^{i},c^{<\pi_{i}},t\right]\right)}}
\end{align*}
where $\left[a,b\right]$ denotes the concatenation of $a$ and $b$.
The MLP is a feed-forward neural network with $L$ layers, where each
layer $\ell$ ($1\leq\ell\leq L$) has finite weights $W_{\ell}$,
which can be enforced via weight regularization during training, and
uses a 1-Lipschitz continuous activation function $\sigma_{\ell}$
(e.g., sigmoid, tanh, ReLU, ELU). According to the theoretical result
in \cite{kim2021lipschitz} (Corollary 2.1), the Lipschitz constant
of the MLP is given by: 

\begin{align*}
\text{Lip}\left(\text{MLP}\right) & =\text{Lip}\left(W_{L}\circ\sigma_{L-1}\circ W_{L-1}\circ...\circ\sigma_{1}\circ W_{1}\right)\\
 & \leq\prod_{\ell=1}^{L}\text{Lip}\left(W_{\ell}\right)
\end{align*}
where $\text{Lip}\left(f\right)$ denotes the Lipschitz constant of
the function $f$. In the case $f$ is a matrix $W\in\mathbb{R}^{m\times n}$,
$\text{Lip}\left(W\right)=\left\Vert W\right\Vert _{2}:=\sup_{\left\Vert a\right\Vert _{2}=1}\left\Vert Wa\right\Vert _{2}$.
Since the matrices $W_{1}$, $W_{2}$,..., $W_{L}$ have finite values,
their Lipschitz constants are finite. This implies that the Lipschitz
constant of the MLP is finite. In other words, $v_{t}^{i}$ is a Lipschitz
continuous function of $z_{t}^{i}$.

\subsection{Algorithms}

\paragraph{Observational/Interventional generation}

To generate observational/interventional samples, we first sample
from the latent distribution $u\sim p\left(u\right)$. In case of
doing intervention $p\left(x|\text{do}\left(x^{i}=\alpha\right)\right)$,
we change the value of latent by doing $x=\text{Predict\ensuremath{\left(u\right)},}\;x^{i}=\alpha,\;u^{i}=\text{Abduct}\left(x\right)^{i}$.
However, if we have access to observational data, we can skip the
prediction part of intervention. Then, we predict the value of exogenous
variables based on generated latent $x^{\text{int}}=\text{Predict\ensuremath{\left(u\right)}}$.
See Alg.~\ref{alg:Observational-Sampling} for observational sampling
algorithm, Alg.~\ref{alg:Interventional-Sampling} for interventional
sampling algorithm. 

\paragraph{Counterfactual generation}

The counterfactual generation process is the same with interventional
generation except that instead of sample the latent from latent distribution
$p\left(u\right)$, we acquire latent from factual endogenous values
$u=\text{Abduct}\left(x^{\text{f}}\right)$. See Alg.~\ref{alg:Counterfactual-Sampling}
for counterfactual sampling algorithm.

\begin{algorithm}
\caption{Observational Sampling\label{alg:Observational-Sampling}}

\begin{algorithmic}[1] 

\State Sample $u \sim P_u$
\State $x = \text{Predict}(u)$
\State $\text{Return } x$
\end{algorithmic}
\end{algorithm}

\begin{algorithm}
\caption{Interventional Sampling\label{alg:Interventional-Sampling}}

\textbf{Input}: Intervention $p\left(x|\text{do}\left(x^{i}=\alpha\right)\right)$

\begin{algorithmic}[1] 

\State Sample $u \sim P_u$
\State $x = \text{Predict}(u)$
\State $x^i = \alpha$
\State $u^i = \text{Abduct}(x)^i$
\State $x = \text{Predict}(u^i)$
\State $\text{Return } x$
\end{algorithmic}
\end{algorithm}

\begin{algorithm}
\caption{Counterfactual Sampling\label{alg:Counterfactual-Sampling}}

\textbf{Input}: Intervention $p\left(x^{\text{cf}}|\text{do}\left(x^{i}=\alpha\right),x^{\text{f}}\right)$

\begin{algorithmic}[1] 

\State $u = \text{Abduct}(x^\text{f}))$
\State $x^{i,\text{f}} = \alpha$
\State $u^i = \text{Abduct}(x^{\text{f}})^i$
\State $x^{\text{cf}} = \text{Predict}(u)$
\State $\text{Return } x^{\text{cf}}$
\end{algorithmic}
\end{algorithm}

\section{Data denoising }

We consider another training strategy in which we try to predict $x=z_{0}$
from $z_{t}$ and $x^{\text{<\ensuremath{\pi_{i}}}}$ via objective
function:
\[
\mathcal{L}\left(\theta\right)=\Expect_{t}\Expect_{x,u}\left\Vert v_{\theta}\left(z_{t},x,t\right)-x\right\Vert _{2}^{2}+\lambda\left\Vert \theta\right\Vert _{2}^{2}.
\]
Velocity of the ODE is approximated by the equation:
\[
v_{t}=\frac{z_{t}-v_{\theta}\left(z_{t},x,t\right)}{t}.
\]
However, this approximation suffers from the denominator, causing
high variance where $t$ is near 0. Empirically, we cap $t$ between
$\left[5e^{-2},1\right]$. The performance of the method nearly as
good as velocity matching in observational and interventional metrics,
but worse in counterfactual metric.

\section{Causal inference experiments}

In this section, we provide complete experimental setups and provide
additional results that can be shown in the main paper due to page
limitation.

\subsection{Training setup}

\paragraph{Hardware}

All experiment ran on 8 cores of Intel(R) Xeon(R) Gold 6248 CPU with
32GB RAM. For scalability experiment, all models are evaluated sequentially
to ensure fairness. 

\paragraph{Model}

We use the Masked Autoregressive Neural Network followed by a simple
MLP. Our experiments show that even with more complex architectures
like U-net, the performance improvement over the MLP is minimal. Another
alternative, the affine transformation, offers faster computation
but is limited in its approximation capabilities. Since our S-CDM
remains a TMI map regardless of the model architecture, we chose MLP,
which balances computational efficiency and approximation power. For
encoding the time step $t$, we experimented with two methods: simple
concatenation to the input and positional embedding. Both approaches
yielded equivalent performance.

\paragraph{Training}

For each model, we ran five experiments and average over runs. We
randomly initialize the parameters using a uniform distribution to
ensure diversity and minimize initialization bias. We generate 50000
samples for training, 5000 for evaluation and 5000 for testing. For
the optimization, we use the Adam optimizer with a learning rate of
0.001 and apply a decay factor of 0.95 to the learning rate if it
remains at a plateau for more than 10 epochs, batch size of 2048,
and train for 900 epochs. Abduction and prediction process are done
using 50 steps.

\paragraph{Datasets}

We use two types of SCMs: additive noise and non-additive noise. The
exogenous distributions are mainly standard Gaussian $\mathcal{N}\left(0,1\right)$
except for Large Backdoor. We define the function $t\left(x\right)=\begin{cases}
\text{tanh}\left(x\right)-1 & \text{if}\ x\geq0\\
\text{tanh}\left(x\right)+1 & \text{if}\ x<0
\end{cases}$. The structural equations for them are defined as below.
\begin{enumerate}
\item Triangle, Nonlinear\textbf{
\begin{align*}
f^{1}\left(u^{1}\right) & =u^{1}\\
f^{2}\left(x^{1},u^{2}\right) & =2\cdot\left(x^{1}\right)^{2}+u^{2}\\
f^{3}\left(x^{1},x^{2},u^{3}\right) & =\frac{20}{1+\exp\left(-\left(x^{2}\right)^{2}+x^{1}\right)}+u^{3}
\end{align*}
}
\item Simpson, Nonlinear
\begin{align*}
f^{1}\left(u^{1}\right) & =u^{1}\\
f^{2}\left(x^{1},u^{2}\right) & =\text{softplus}\left(1-x^{1}\right)+\sqrt{\frac{3}{20}}\cdot u^{2}\\
f^{3}\left(x^{1},x^{2},u^{3}\right) & =\text{tanh}\left(2\cdot x^{2}\right)+\frac{3}{2}\cdot x^{1}-1-\text{tanh}\left(u^{3}\right)\\
f^{4}\left(x^{3},u^{4}\right) & =\frac{\left(x^{3}-4\right)}{5}+3+\frac{1}{\sqrt{10}}\cdot u^{4}
\end{align*}
\item Simpson, Nonlinear2\textbf{
\begin{align*}
f^{1}\left(u^{1}\right) & =t\left(u^{1}\right)\\
f^{2}\left(x^{1},u^{2}\right) & =x^{1}+t\left(u^{2}\right)\\
f^{3}\left(x^{1},x^{2},u^{3}\right) & =2\cdot\left(x^{1}\right)^{2}+\left(x^{2}\right)^{3}+t\left(u^{3}\right)\\
f^{4}\left(x^{3},u^{4}\right) & =x^{3}+t\left(u^{4}\right)
\end{align*}
}
\item Diamond, Nonlinear
\begin{align*}
f^{1}\left(u^{1}\right) & =u^{1}\\
f^{2}\left(x^{1},u^{2}\right) & =\left(x^{1}\right)^{2}+\frac{u^{2}}{2}\\
f^{3}\left(x^{1},x^{2},u^{3}\right) & =\left(x^{2}\right)^{2}-\frac{2}{1+\text{exp}\left(-x^{1}\right)}+\frac{u^{3}}{2}\\
f^{4}\left(x^{2},x^{3},u^{4}\right) & =\frac{x^{3}}{\left|x^{2}+2\right|+x^{3}+0.5}+\frac{u^{4}}{10}
\end{align*}
\item Diamond, Nonadditive
\begin{align*}
f^{1}\left(u^{1}\right) & =u^{1}\\
f^{2}\left(x^{1},u^{2}\right) & =\frac{\sqrt{\left|x^{1}\right|}\cdot\left|u^{2}+0.1\right|}{2}+\left|x^{1}\right|+\frac{u^{2}}{5}\\
f^{3}\left(x^{1},x^{2},u^{3}\right) & =\frac{1}{1+\left|u^{3}+0.5\right|\cdot\text{exp}\left(x^{1}-x^{2}\right)}\\
f^{4}\left(x^{2},x^{3},u^{4}\right) & =\left(x^{2}+x^{3}+\frac{u^{4}}{4}-7\right)^{2}-20
\end{align*}
\item Y, Nonlinear
\begin{align*}
f^{1}\left(u^{1}\right) & =u^{1}\\
f^{2}\left(u^{2}\right) & =u^{2}\\
f^{3}\left(x^{1},x^{2},u^{3}\right) & =\left(-\left(x^{2}\right)^{2}+\frac{4}{1+\text{exp}\left(-x^{1}-x^{2}\right)}+\frac{u^{3}}{2}\right)/1.83\\
f^{4}\left(x^{2},x^{3},u^{4}\right) & =\left(\frac{20}{1+\text{exp}\left(\frac{\left(x^{3}\right)^{2}}{2}-x^{3}\right)}+u^{4}\right)/3.26
\end{align*}
\item Y, Nonadditive
\begin{align*}
f^{1}\left(u^{1}\right) & =u^{1}\\
f^{2}\left(u^{2}\right) & =u^{2}\\
f^{3}\left(x^{1},x^{2},u^{3}\right) & =\left(x^{1}-2\cdot x^{2}-2\right)\cdot\left(\left|u^{3}\right|+0.2\right)\\
f^{4}\left(x^{3},u^{4}\right) & =\left(\text{cos}\left(x^{3}\right)+\frac{u^{4}}{2}\right)^{2}
\end{align*}
\item Large Backdoor, Nonlinear
\begin{align*}
l\left(x,y\right) & =\text{softplus}\left(x+1\right)+\text{softplus}\left(0.5+y\right)-3.0\\
f^{1}\left(u^{1}\right) & =\text{softplus}\left(1.8\cdot u^{1}\right)-1\\
f^{2}\left(x^{2},u^{2}\right) & =0.25\cdot u^{2}+1.5\cdot l\left(x^{1},0\right)\\
f^{3}\left(x^{1},u^{3}\right) & =l\left(x^{1},u^{3}\right)\\
f^{4}\left(x^{2},u^{4}\right) & =l\left(x^{2},u^{4}\right)\\
f^{5}\left(x^{3},u^{5}\right) & =l\left(x^{3},u^{5}\right)\\
f^{6}\left(x^{4},u^{6}\right) & =l\left(x^{4},u^{6}\right)\\
f^{7}\left(x^{5},u^{7}\right) & =l\left(x^{5},u^{7}\right)\\
f^{8}\left(x^{6},u^{8}\right) & =0.3\cdot u^{8}+\left(\text{softplus}\left(x^{6}+1\right)-1\right)\\
f^{9}\left(x^{7},x^{8},u^{9}\right) & =\boldsymbol{\text{CDF}}^{\boldsymbol{-1}}\left(-\text{softplus}\left(\frac{1.3\cdot x^{7}+x^{8}}{3}+1\right)+2,0.6,u^{9}\right)
\end{align*}
, where $\text{CDF}\left(\mu,b,x\right)$ is the quantile function
of a Laplace distribution with location $\mu$, scale $b$, evaluated
at $x$.
\item Large Backdoor, Nonadditive
\begin{align*}
f^{1}\left(u^{1}\right) & =u^{1}\\
f^{2}\left(x^{2},u^{2}\right) & =\left(x^{1}\right)^{2}+\frac{u^{2}}{2}\\
f^{3}\left(x^{1},u^{3}\right) & =\left(x^{1}\right)^{2}-\frac{2}{1+\text{exp}\left(-x^{1}\right)}+\frac{u^{3}}{2}\\
f^{4}\left(x^{2},u^{4}\right) & =\frac{u^{4}}{\left|x^{2}+2\right|+x^{2}+0.5}+\frac{u^{4}}{10}\\
f^{5}\left(x^{3},u^{5}\right) & =\left(x^{3}-u^{5}\right)^{2}-\frac{2}{1+\text{exp}\left(-x^{3}\right)+u^{5}}+\frac{u^{5}}{20}\\
f^{6}\left(x^{4},u^{6}\right) & =\left(x^{4}+u^{6}\right)^{2}-\frac{2}{1+\text{exp}\left(-u^{6}\right)+x^{4}}+\frac{u^{6}}{2}\\
f^{7}\left(x^{5},u^{7}\right) & =\left(x^{5}+u^{7}\right)^{2}-\frac{u^{7}}{1+\text{exp}\left(-u^{7}\right)+x^{5}}+\frac{u^{7}}{2}\\
f^{8}\left(x^{6},u^{8}\right) & =\left(x^{6}-u^{8}\right)^{2}-\frac{u^{8}}{1+\text{exp}\left(-u^{8}\right)+x^{6}}+\frac{u^{8}}{2}\\
f^{9}\left(x^{7},x^{8},u^{9}\right) & =\frac{\left(x^{7}-x^{8}-u^{9}\right)^{2}}{10}-\frac{x^{8}}{5+\frac{\text{exp}\left(-x^{7}\right)}{u^{9}}}+\frac{x^{8}}{2}
\end{align*}
\end{enumerate}
\begin{figure}
\begin{centering}
\subfloat[Triangle]{\includegraphics[width=0.12\paperwidth]{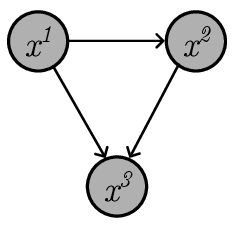}

}\hspace{0.5cm}\subfloat[Simpson]{\includegraphics[width=0.17\paperwidth]{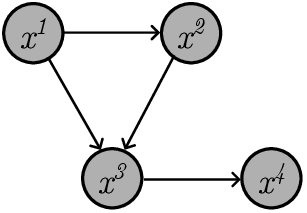}

}\hspace{0.5cm}\subfloat[Y]{\includegraphics[width=0.11\paperwidth]{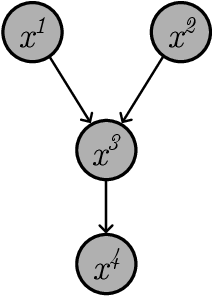}

}\vspace{1cm}
\par\end{centering}
\begin{centering}
\subfloat[Diamond]{\includegraphics[width=0.17\paperwidth]{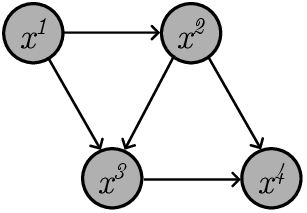}

}\hspace{0.5cm}\subfloat[Large Backdoor]{\includegraphics[width=0.23\paperwidth]{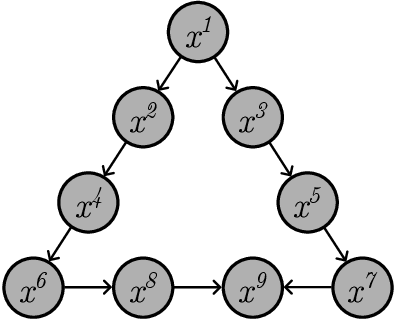}

}
\par\end{centering}
\caption{Graph structures of synthetic datasets}
\end{figure}

\subsection{Result}

Table \ref{tab:full_causal_inference_result} presents the performance
of the proposed models, S-CFM and P-CFM, as well as previous methods
across all datasets. VACA demonstrates poor performance across all
metrics and datasets. Although CausalNF shows competitive results
compared to our methods, CausalNF-MAF performs poorly on both observational
and interventional MMD metrics. Additionally, CausalCF-NSF shows low
accuracy in approximating counterfactuals. In contrast, our proposed
methods consistently show strong performance, often achieving the
best or competitive results in all metrics compared to CausalNF-NSF
and CausalNF-MAF. 

Fig. \ref{fig:scatter-y-nadd-int} presents scatter plots of interventional
samples generated from the ground-truth SCM and various methods. The
results indicate that while CausalNF-MAF struggles to approximate
the interventional distribution, both CausalNF-NSF and our proposed
methods, S-CFM and P-CFM, demonstrate high accuracy in approximating
the interventional distribution.

Fig. \ref{fig:pairplot-diamond-nonlinear} presents qualitative results
of the P-CFM model in capturing both observational and interventional
distributions for the Diamond, Nonlinear dataset. In this plot, true
distributions/samples are shown in blue, while P-CFM predicted distributions/samples
are depicted in orange. Fig. \ref{fig:obs-pairplot-diamond-nonlinear}
clearly illustrates that the model successfully captures the correlations
among all variables in the observational distribution. Fig. \ref{fig:int-pairplot-diamond-nonlinear}
shows the interventional distributions resulting from an intervention
on the 25-th empirical percentile of $x^{2}$, $\text{do}\left(x^{2}=0.08\right)$.
It is evident that P-CFM accurately learns the distribution of descendant
variables and effectively eliminates any dependency between the ancestors
of the intervened variable. Fig. \ref{fig:pairplot-y-nonadditive}
demonstrates a comparable result for the Y, Nonadditive dataset, further
validating the model's ability to handle different types of data and
structural equations.

\begin{table*}
\begin{centering}
\begin{tabular}{>{\raggedright}m{1.8cm}lrrrrr}
\toprule 
\multirow{2}{1.8cm}{\centering{}SCM} & \multirow{2}{*}{Metric} & S-CFM & P-CFM & CausalNF-NSF & CausalNF-MAF & VACA\tabularnewline
\cmidrule{3-7} \cmidrule{4-7} \cmidrule{5-7} \cmidrule{6-7} \cmidrule{7-7} 
 &  & $\left(\times10^{-2}\right)$ & $\left(\times10^{-2}\right)$ & $\left(\times10^{-2}\right)$ & $\left(\times10^{-2}\right)$ & $\left(\times10^{-2}\right)$\tabularnewline
\midrule
\multirow{3}{1.8cm}{\centering{}Triangle\linebreak{}
NADD} & Obs. MMD & $0.67_{0.07}$ & \emph{$\mathit{0.65_{0.15}}$} & $\boldsymbol{0.43_{0.08}}$ & $2.03_{0.14}$ & $2.37_{0.25}$\tabularnewline
 & Int. MMD & $3.21_{0.34}$ & \textbf{$\mathit{3.16_{0.24}}$} & $\boldsymbol{1.86_{0.09}}$ & $13.31_{0.29}$ & $15.53_{1.24}$\tabularnewline
 & CF. MSE & $\boldsymbol{6.53_{0.34}}$ & $\mathit{9.59_{3.53}}$ & $41.90_{11.74}$ & $9.91_{1,02}$ & $146.73_{16.29}$\tabularnewline
\midrule
\multirow{3}{1.8cm}{\centering{}Simpson\linebreak{}
NLIN} & Obs. MMD & $\boldsymbol{0.57_{0.05}}$ & \textbf{$\mathit{0.61_{0.05}}$} & $0.62_{0.05}$ & $3.54_{0.21}$ & $2.44_{0.22}$\tabularnewline
 & Int. MMD & $\boldsymbol{0.58_{0.01}}$ & \textbf{$\mathit{0.60_{0.04}}$} & $0.61_{0.01}$ & $4.34_{0.10}$ & $2.55_{0.32}$\tabularnewline
 & CF. MSE & \textbf{$\boldsymbol{0.72_{0.18}}$} & $\mathit{0.73_{0.18}}$ & $1.98_{0.86}$ & $1.29_{0.20}$ & $15.28_{1.25}$\tabularnewline
\midrule
\multirow{3}{1.8cm}{\centering{}Simpson\linebreak{}
NLIN2} & Obs. MMD & $\mathit{0.44_{0.07}}$ & $\mathit{0.44_{0.06}}$ & \textbf{$\boldsymbol{0.41_{0.03}}$} & $0.85_{0.28}$ & $2.44_{0.22}$\tabularnewline
 & Int. MMD & $\boldsymbol{0.43_{0.02}}$ & $\mathit{0.44_{0.02}}$ & $0.46_{0.02}$ & $0.85_{0.06}$ & $2.55_{0.32}$\tabularnewline
 & CF. MSE & $\boldsymbol{0.06_{0.01}}$ & $0.16_{0.06}$ & $3.34_{0.70}$ & $\mathit{0.09_{0.02}}$ & $15.28_{1.25}$\tabularnewline
\midrule
\multirow{3}{1.8cm}{\centering{}Diamond\linebreak{}
NLIN} & Obs. MMD & $\boldsymbol{0.36_{0.02}}$ & $\mathit{0.38_{0.06}}$ & $0.44_{0.04}$ & $2.26_{4.15}$ & $5.23_{0.35}$\tabularnewline
 & Int. MMD & \textbf{$\mathit{0.41_{0.05}}$} & \textbf{$\mathit{0.41_{0.06}}$} & \textbf{$\boldsymbol{0.34_{0.02}}$} & $4.49_{8.95}$ & $26.64_{0.63}$\tabularnewline
 & CF. MSE & \textbf{$\mathit{0.04_{0.04}}$} & $\boldsymbol{0.03_{0.01}}$ & $16.06_{2.07}$ & $197.27_{438.22}$ & $60.62_{3.45}$\tabularnewline
\midrule
\multirow{3}{1.8cm}{\centering{}Diamond\linebreak{}
NADD} & Obs. MMD & \textbf{$\boldsymbol{0.5}\boldsymbol{1_{0.07}}$} & $0.60_{0.07}$ & \textbf{$\mathit{0.54_{0.03}}$} & $0.62_{0.16}$ & $2.96_{0.35}$\tabularnewline
 & Int. MMD & \textbf{$\boldsymbol{0.71_{0.09}}$} & $1.14_{0.39}$ & $\mathit{1.03_{0.08}}$ & $1.96_{0.07}$ & $17.91_{1.12}$\tabularnewline
 & CF. MSE & \textbf{$\boldsymbol{35.31_{3.00}}$} & $44.55_{7.36}$ & $\mathit{44.02_{3.71}}$ & $45.95_{3.04}$ & $40.04_{0.73}$\tabularnewline
\midrule 
\multirow{3}{1.8cm}{\centering{}Y\linebreak{}
NLIN} & Obs. MMD & $\boldsymbol{0.41_{_{0.01}}}$ & \textbf{$\mathit{0.42_{0.02}}$} & $0.47_{0.07}$ & $0.42_{0.04}$ & $2.95_{0.06}$\tabularnewline
 & Int. MMD & $\boldsymbol{0.40_{0.02}}$ & \textbf{$\mathit{0.41_{0.02}}$} & $0.42_{0.03}$ & $0.43_{0.02}$ & $3.99_{0.80}$\tabularnewline
 & CF. MSE & $\boldsymbol{0.04_{0.00}}$ & \textbf{$\mathit{0.05_{0.00}}$} & $0.05_{0.00}$ & $0.18_{0.11}$ & $5.51_{1.35}$\tabularnewline
\midrule 
\multirow{3}{1.8cm}{\centering{}Y\linebreak{}
NADD} & Obs. MMD & $\boldsymbol{0.46_{0.05}}$ & $\boldsymbol{0.46_{0.05}}$ & \textbf{$\mathit{0.55_{0.07}}$} & $1.01_{0.05}$ & $1.68_{0.08}$\tabularnewline
 & Int. MMD & $\boldsymbol{0.48_{0.03}}$ & \textbf{$\mathit{0.51_{0.04}}$} & $1.75_{0.70}$ & $1.56_{0.13}$ & $4.44_{0.20}$\tabularnewline
 & CF. MSE & $\boldsymbol{28.58_{1.04}}$ & \textbf{$\mathit{29.67_{0.84}}$} & $37.31_{3.32}$ & $30.27_{0.78}$ & $32.45_{0.72}$\tabularnewline
\midrule
\multirow{3}{1.8cm}{\centering{}LargeBD\linebreak{}
NLIN} & Obs. MMD & \textbf{$\mathit{0.40_{0.03}}$} & $\boldsymbol{0.39_{0.02}}$ & $0.42_{0.04}$ & $0.97_{0.06}$ & $1.34_{1.14}$\tabularnewline
 & Int. MMD & $\boldsymbol{0.41_{0.02}}$ & $0.48_{0.03}$ & \textbf{$\mathit{0.43_{0.03}}$} & $0.99_{0.03}$ & $0.98_{0.06}$\tabularnewline
 & CF. MSE & $\boldsymbol{0.02_{0.00}}$ & \textbf{$\mathit{0.03_{0.00}}$} & $0.04_{0.00}$ & $0.04_{0.00}$ & $0.71_{0.01}$\tabularnewline
\midrule
\multirow{3}{1.8cm}{\centering{}LargeBD\linebreak{}
NADD} & Obs. MMD & $\boldsymbol{0.50_{0.04}}$ & \textbf{$\mathit{0.52_{0.05}}$} & \textbf{$\mathit{0.52_{0.03}}$} & $0.90_{0.12}$ & $150.46_{1.11}$\tabularnewline
 & Int. MMD & \textbf{$\mathit{0.56_{0.05}}$} & $0.60_{0.05}$ & $\boldsymbol{0.49_{0.02}}$ & $1.23_{0.07}$ & $187.99_{0.49}$\tabularnewline
 & CF. MSE & \textbf{$\mathit{2.62_{0.40}}$} & $5.47_{0.58}$ & $61.80_{18.83}$ & $\boldsymbol{2.41_{0.47}}$ & $333.85_{309.92}$\tabularnewline
\midrule
\multirow{3}{1.8cm}{\centering{}LargeLadder\linebreak{}
NLIN} & Obs. MMD & $\boldsymbol{0.63_{0.01}}$ & $\textit{0.64}_{\textit{0.00}}$ & $0.66_{0.04}$ & $0.65_{0.05}$ & -\tabularnewline
 & Int. MMD & $\boldsymbol{0.64_{0.01}}$ & $\textit{0.65}_{\textit{0.00}}$ & $0.66_{0.00}$ & $\textit{0.65}_{\textit{0.00}}$ & -\tabularnewline
 & CF. MSE & $\textit{42.71}_{\textit{3.12}}$ & $46.79_{6.67}$ & $72.63_{5.4}$ & $\boldsymbol{38.46_{1.30}}$ & -\tabularnewline
\midrule
\multirow{3}{1.8cm}{\centering{}LargeLadder\linebreak{}
NADD} & Obs. MMD & $\boldsymbol{0.58_{0.06}}$ & $\textit{0.61}_{\textit{0.02}}$ & $0.62_{0.03}$ & $0.71_{0.05}$ & -\tabularnewline
 & Int. MMD & $\boldsymbol{0.62_{0.00}}$ & $\textit{0.63}_{\textit{0.00}}$ & $\boldsymbol{0.62_{0.00}}$ & $0.68_{0.00}$ & -\tabularnewline
 & CF. MSE & $\boldsymbol{27.85_{1.39}}$ & $\textit{29.09}_{\textit{0.29}}$ & $36.23_{3.49}$ & $37.42_{1.94}$ & -\tabularnewline
\bottomrule
\end{tabular}\caption{Mean and standard deviation of eight SCMs with non-linear and non-additive
structural equations of proposed methods and existing methods. The
values are scaled by 100 for interpretable. Each method is evaluated
over 5 random initializations of the model and training data. \label{tab:full_causal_inference_result}}
\par\end{centering}
\end{table*}

\begin{figure}
\subfloat[Observational distribution\label{fig:obs-pairplot-diamond-nonlinear}]{\begin{centering}
\includegraphics[width=0.47\textwidth]{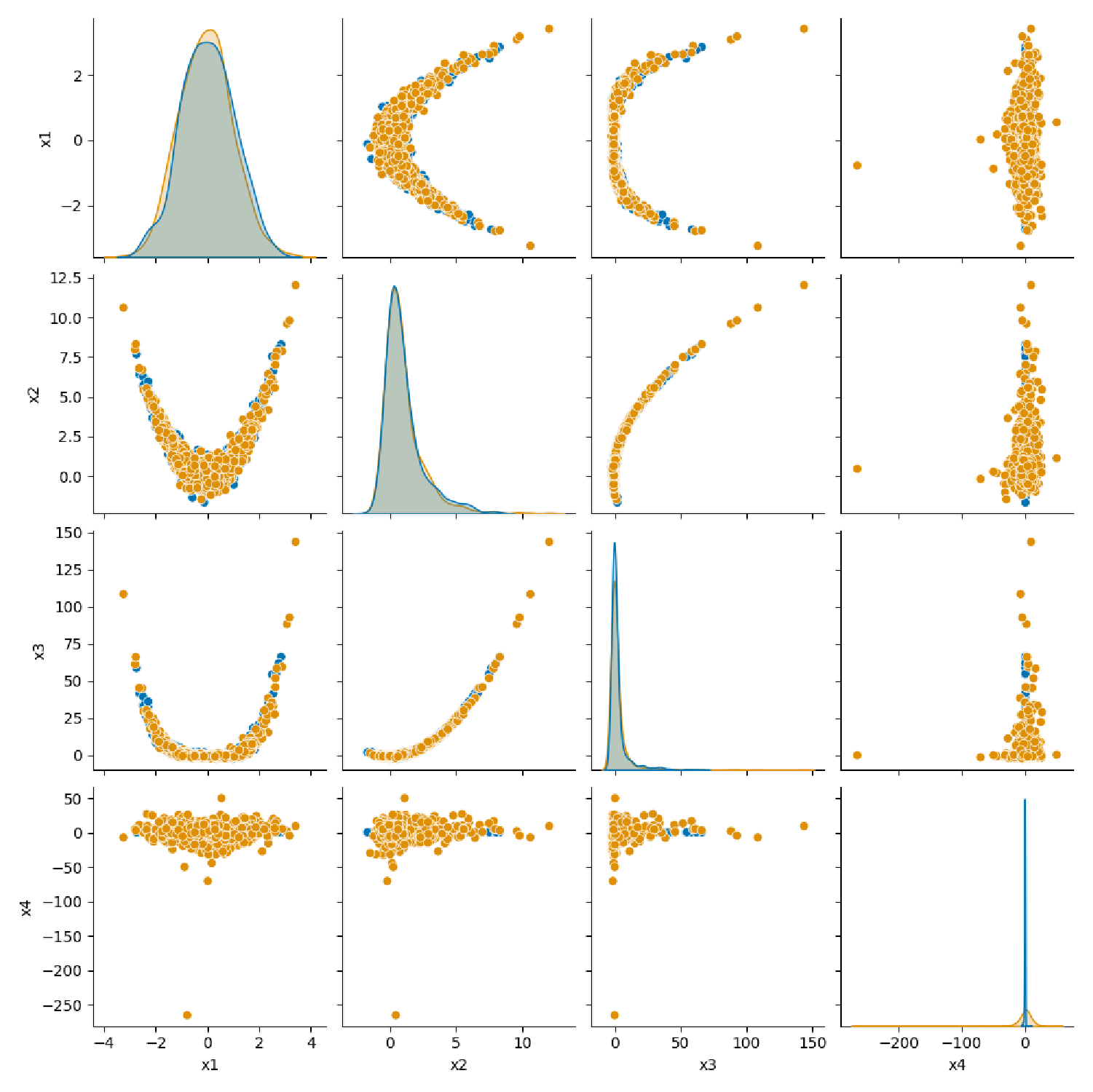}
\par\end{centering}

}\hspace{0.5cm}\subfloat[Interventional distribution\label{fig:int-pairplot-diamond-nonlinear}]{\begin{centering}
\includegraphics[width=0.47\textwidth]{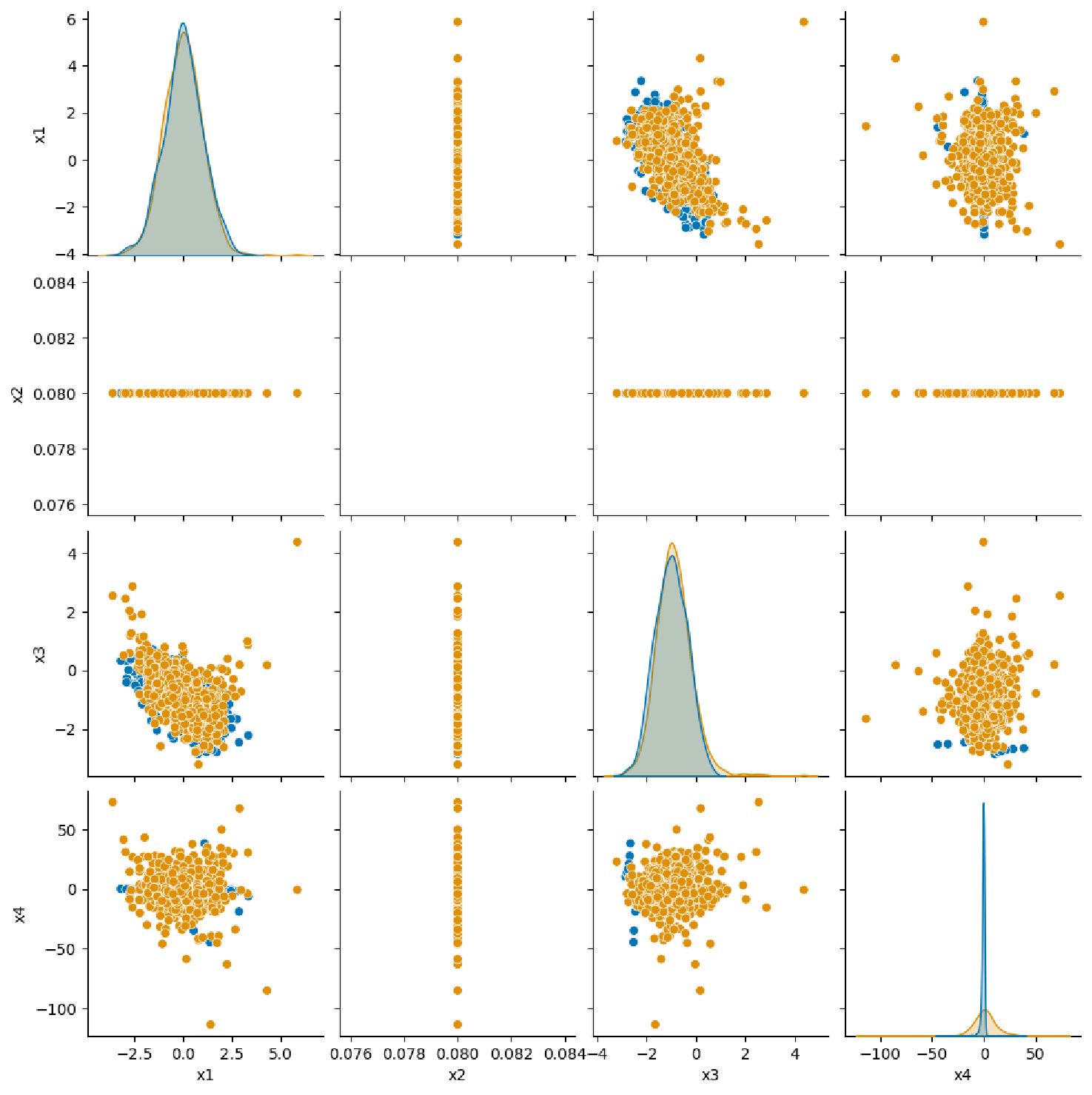}
\par\end{centering}
}

\caption{Pair plot of true (blue) and P-CFM predicted (orange) data for the
Diamond, Nonlinear dataset. True and predicted observational samples
are displayed on the left. True and predicted interventional samples
under $\text{do}\left(x^{2}=0.08\right)$ are shown on the right.
\label{fig:pairplot-diamond-nonlinear}}
\end{figure}

\begin{figure}
\subfloat[Observational distribution]{\begin{centering}
\includegraphics[width=0.47\textwidth]{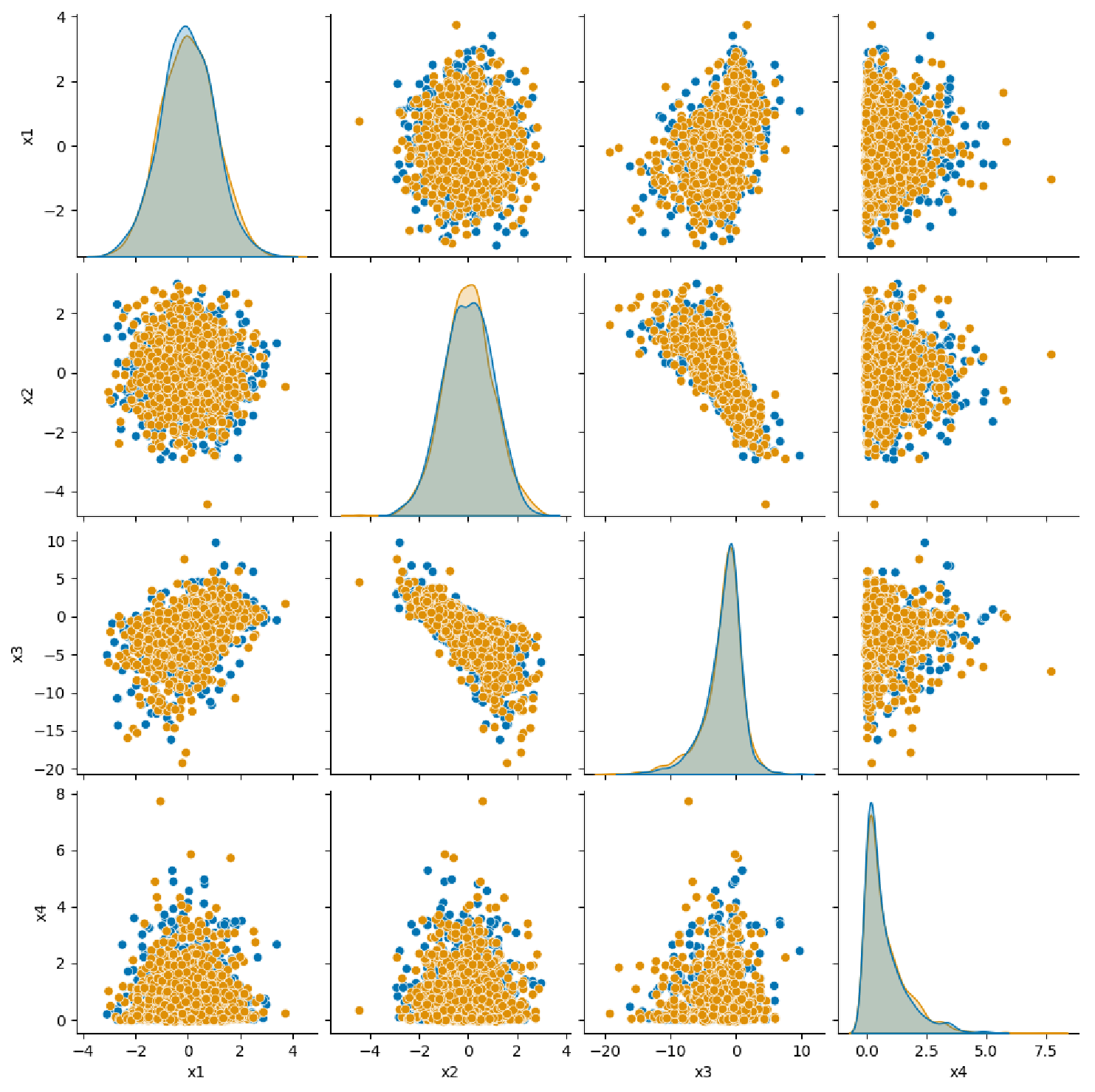}
\par\end{centering}
}\hspace{0.5cm}\subfloat[Interventional distribution]{\begin{centering}
\includegraphics[width=0.47\textwidth]{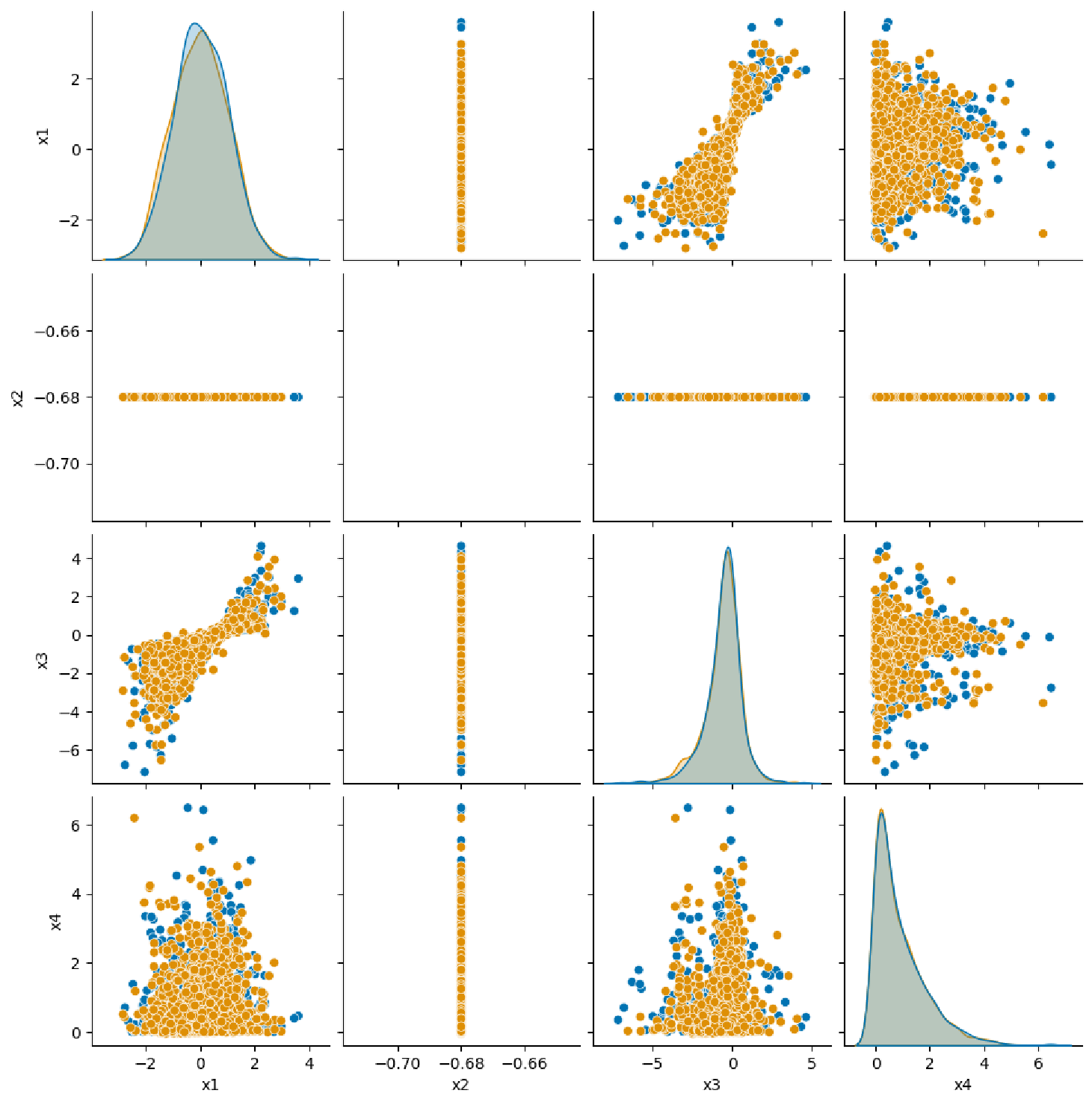}
\par\end{centering}
}

\caption{Pair plot of true (blue) and P-CFM predicted (orange) data for the
Y, Nonadditive dataset. True and predicted observational samples are
displayed on the left. True and predicted interventional samples under
$\text{do}\left(x^{2}=-0.68\right)$ are shown on the right. \label{fig:pairplot-y-nonadditive}}
\end{figure}

\begin{figure}
\subfloat[Scatter plots of $x^{3}$ and $x^{4}$ from the Y, Nonadditive dataset
under the intervention $\text{do}\left(x^{1}=0.68\right)$]{\begin{centering}
\includegraphics[width=1\textwidth]{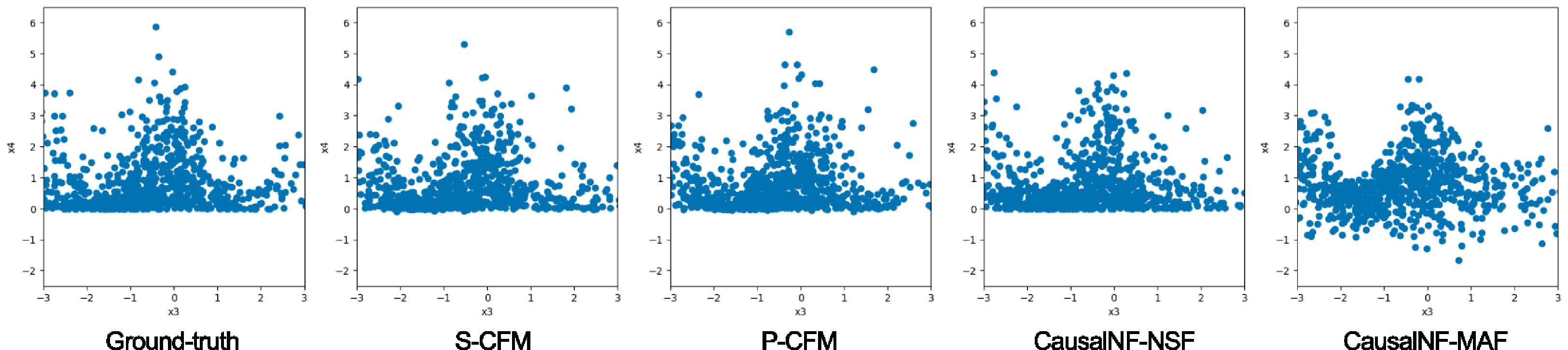}
\par\end{centering}

}\vspace{0.5cm}

\subfloat[Scatter plots of $x^{3}$ and $x^{4}$ from the Y, Nonadditive dataset
under the intervention $\text{do}\left(x^{2}=0.67\right)$]{\begin{centering}
\includegraphics[width=1\textwidth]{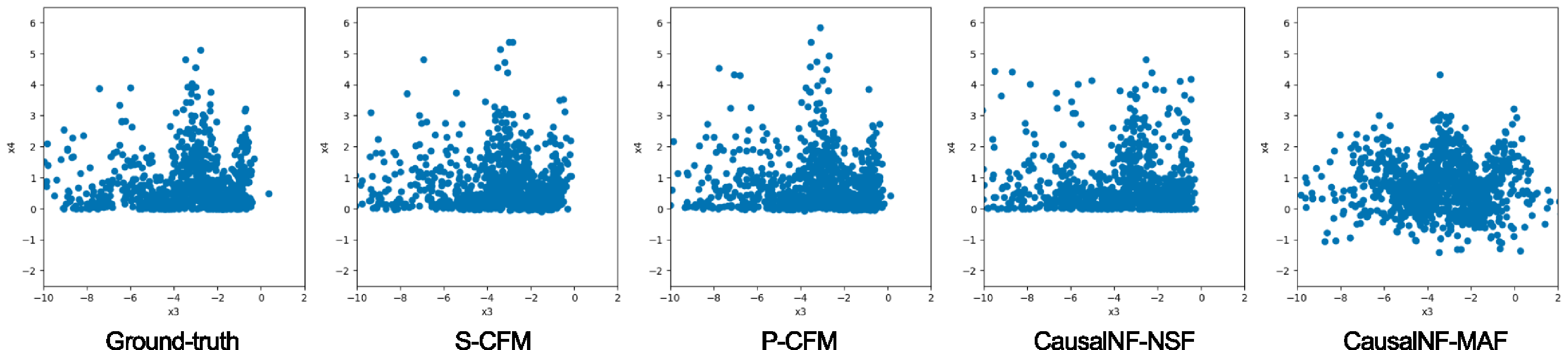}
\par\end{centering}
}\caption{Scatter plots of $x^{3}$ and $x^{4}$ from the Y, Nonadditive dataset,
showing the effects of interventions on $x^{1},x^{2}$, are compared
between the ground-truth and predictions from different methods\label{fig:scatter-y-nadd-int}}

\end{figure}

\pagebreak{}

\section{Visualization}

In this section, we present a visualization that demonstrates the
quality of our method in estimating the counterfactual of an image.

\subsection{Experimental settings}

\paragraph{Dataset}

We aim to model the causal structure of a synthetic dataset generated
from Morpho-MNIST \cite{castro2019morpho}. We extend the previous
use cases of Morpho-MNIST \cite{pmlr-v202-de-sousa-ribeiro23a,Pawlowski2020}
by introducing the Angle variable into the causal structure. This
addition creates a causal graph with a confounder, thereby making
the problem more challenging to solve. We define stroke thickness
as influencing both the brightness and the angle of the digit. Specifically,
a thicker digit results in a brighter and more slanted digit, and
vice versa. Additionally, brightness has a slight positive effect
on the angle of the digit. To generate images based on latent variables,
we use morphological transformations as described in \cite{castro2019morpho}.
Causal graph is depicted in Fig.~\ref{fig:morphomnist}, and the
structural equations are shown below:
\begin{align*}
\text{Thickness} & =t=f^{1}\left(u^{1}\right)=2.5+0.64\cdot u^{1}, & u^{1}\sim\mathcal{N}\left(0,1\right),\\
\text{Intensity} & =i=f^{2}\left(x^{1},u^{2}\right)=191+\text{Sigmoid}\left(0.5\cdot u^{2}+2\cdot x^{1}-5\right)+64, & u^{2}\sim\mathcal{N}\left(0,1\right),\\
\text{Angle} & =s=f^{3}\left(x^{1},x^{2},u^{3}\right)=\frac{2}{3}\cdot\pi+\text{Sigmoid}\left(u^{3}+2\cdot x^{1}-\frac{2\cdot x^{2}}{255}-6\right)-\frac{\pi}{3}, & u^{3}\sim\mathcal{N}\left(0,1\right),\\
\text{Digit} & =y=f^{4}\left(u^{4}\right)=u^{4}, & u^{4}\sim\mathcal{U}\left\{ 0,9\right\} ,\\
\text{Image} & =x=f^{5}\left(t,i,s,y,u^{5}\right)=\text{SetAngle}\left(\text{SetThickness}\left(\text{SetIntensity}\left(u^{5},i\right),t\right),s\right), & u^{5}\sim\text{MNIST},
\end{align*}
where SetAngle, SetThickness and SetIntensity are operators that act
on images, transforming original MNIST images into images with thickness,
intensity and angle consistent with latent variables. Fig. \ref{fig:morphomnist-samples}
shows random samples from the dataset with various values of Thickness,
Intensity and Angle.

\begin{figure}
\begin{centering}
\includegraphics[width=0.3\paperwidth]{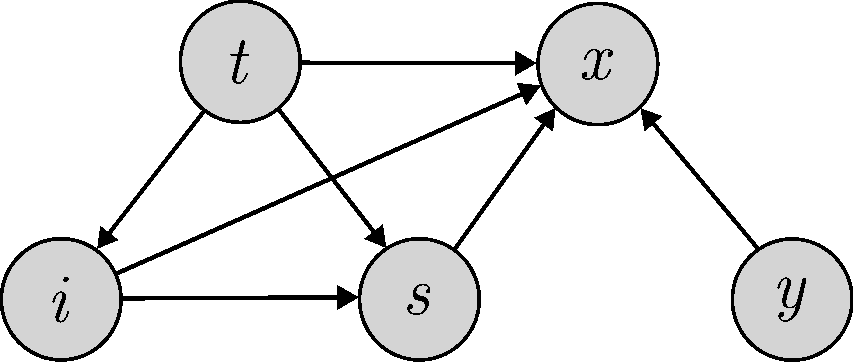}
\par\end{centering}
\caption{Causal graph or Morpho-MNIST, $t$ is Thickness, $i$ is Intensity,
$s$ is Angle, $y$ is Digit class, $x$ is Image\label{fig:morphomnist}}

\end{figure}

\begin{figure}
\centering{}\includegraphics[width=1\textwidth]{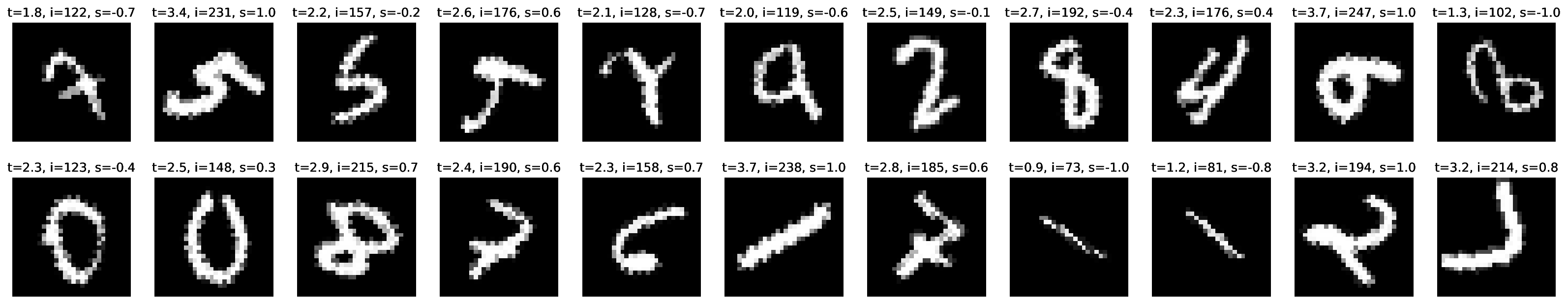}\caption{Random samples from the Morpho-MNIST dataset \label{fig:morphomnist-samples}}
\end{figure}
\begin{figure}
\begin{centering}
\includegraphics[width=1\textwidth]{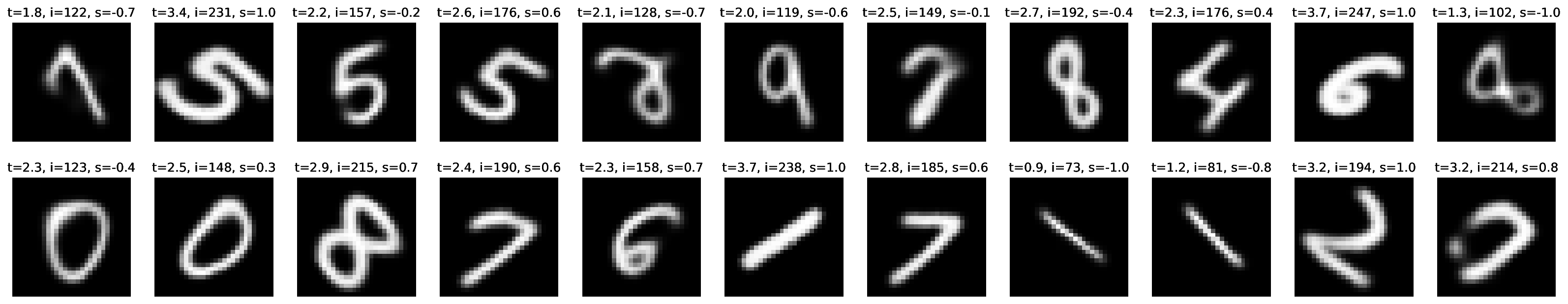}\caption{Reconstructed images generated by the VAE from the Morpho-MNIST dataset
\label{fig:morphomnist-recon}}
\par\end{centering}
\end{figure}

\paragraph{Model}

We use S-CFM for modelling causal graph of Thickness, Intensity and
Angle. We employ a VAE model with a ResNet-18 encoder and decoder
\cite{he2016deep}. To ensure consistency with the latent variables,
we utilize a fully connected neural network to encode the latent variables
(Thickness, Intensity, Angle, and Digit). The encoded information
from the images and latent variables is then concatenated and fed
into the decoder. This design ensures that the ResNet encoder focuses
on encoding the style of the image, while the fully connected neural
network encodes the latent variable information. Fig. \ref{fig:morphomnist-recon}
shows reconstructed images generated by the VAE from random samples
presented in Fig. \ref{fig:morphomnist-samples}. While the VAE model
preserves latent information, it slightly alters the style of the
images.

When estimating the counterfactual of an image, we first calculate
the counterfactual of the latent variables using S-CFM. We then feed
the factual image and counterfactual latents into the VAE. The reconstructed
image generated by the VAE is the counterfactual image we seek.

\subsection{Result}

The result of estimating counterfactual when intervening on Thickness,
Intensity and Angle are shown in Figures \ref{fig:counterfactual-morphomnist-thickness},
\ref{fig:counterfactual-morphomnist-intensity}, and \ref{fig:counterfactual-morphomnist-slant}.
We observe that intervening on Thickness also alters Intensity and
Angle. In contrast, intervening on Intensity leaves Thickness unchanged
and causes only slight changes to Angle. Intervening on Angle keeps
both Thickness and Intensity unchanged. The VAE model effectively
reconstructs counterfactual images. Specifically, a higher thickness
results in a brighter and more slanted digit. When intervening on
Intensity, the counterfactual images are generated without altering
the digit\textquoteright s thickness. Intervening on Angle does not
change the Thickness and Intensity of the digit in the counterfactual
images. In some cases, the VAE model alters the style of the digits
without changing the digit labels. This issue is the problem of the
VAE model while still consistent with latent variables. This experiment
serves as a qualitative verification of our method, demonstrating
its consistency with the causal graph shown in Fig.~\ref{fig:morphomnist}
and its capability for accurate counterfactual estimation. 

\begin{figure}
\begin{centering}
\includegraphics[width=0.7\textwidth]{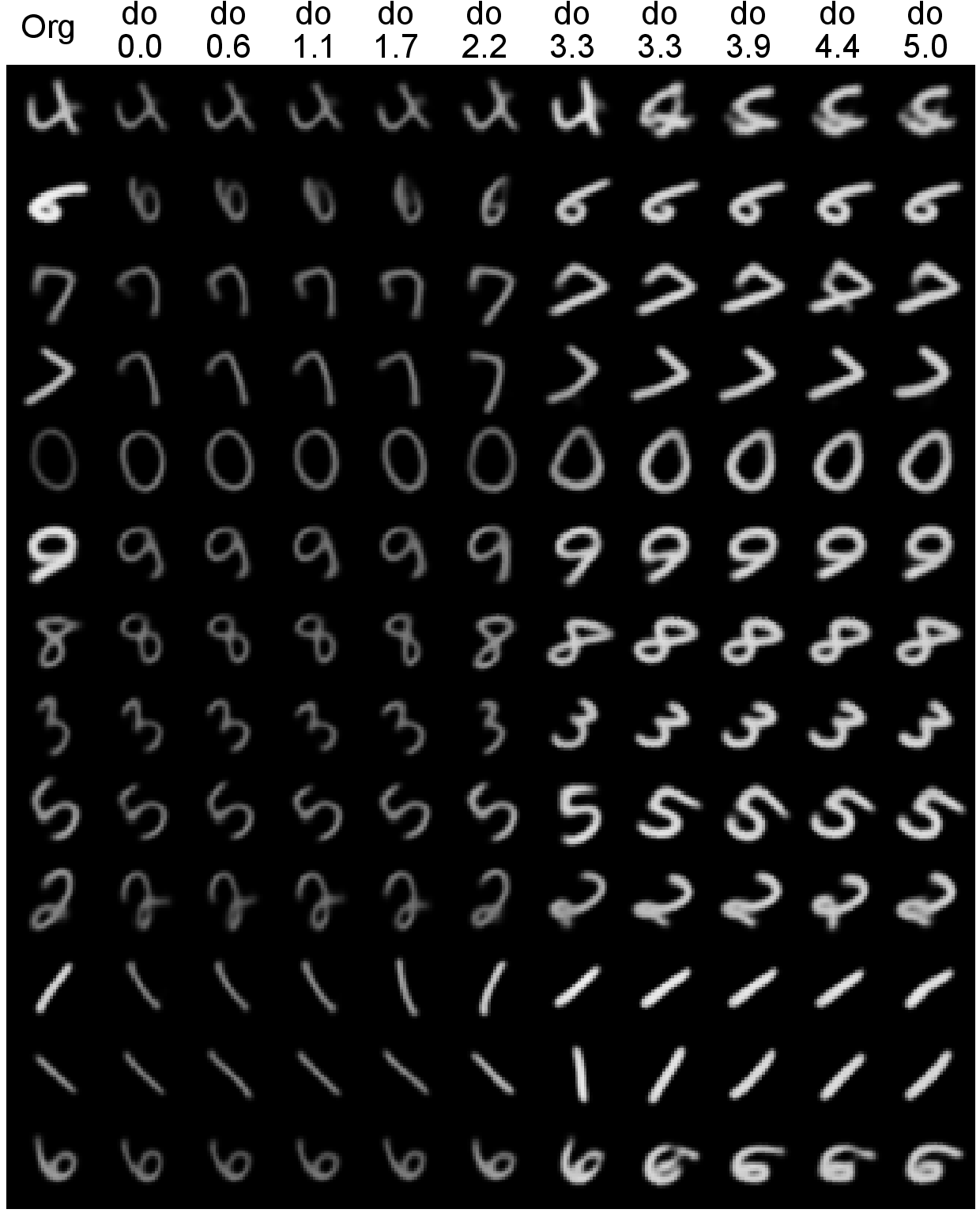}
\par\end{centering}
\caption{Counterfactual when intervening on Thickness\label{fig:counterfactual-morphomnist-thickness}}

\end{figure}

\begin{figure}
\begin{centering}
\includegraphics[width=0.7\textwidth]{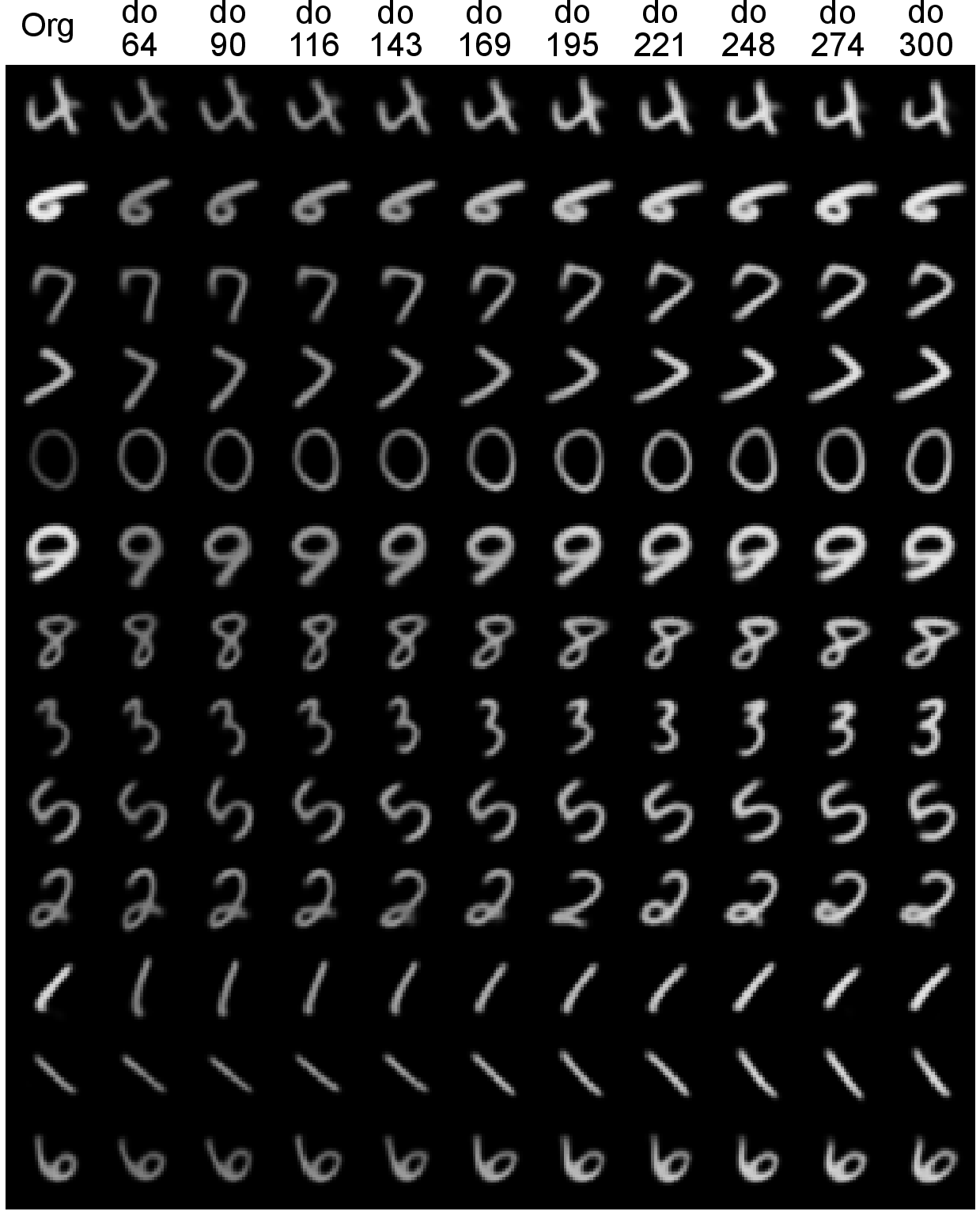}
\par\end{centering}
\caption{Counterfactual when intervening on Intensity\label{fig:counterfactual-morphomnist-intensity}}
\end{figure}

\begin{figure}
\begin{centering}
\includegraphics[width=0.7\textwidth]{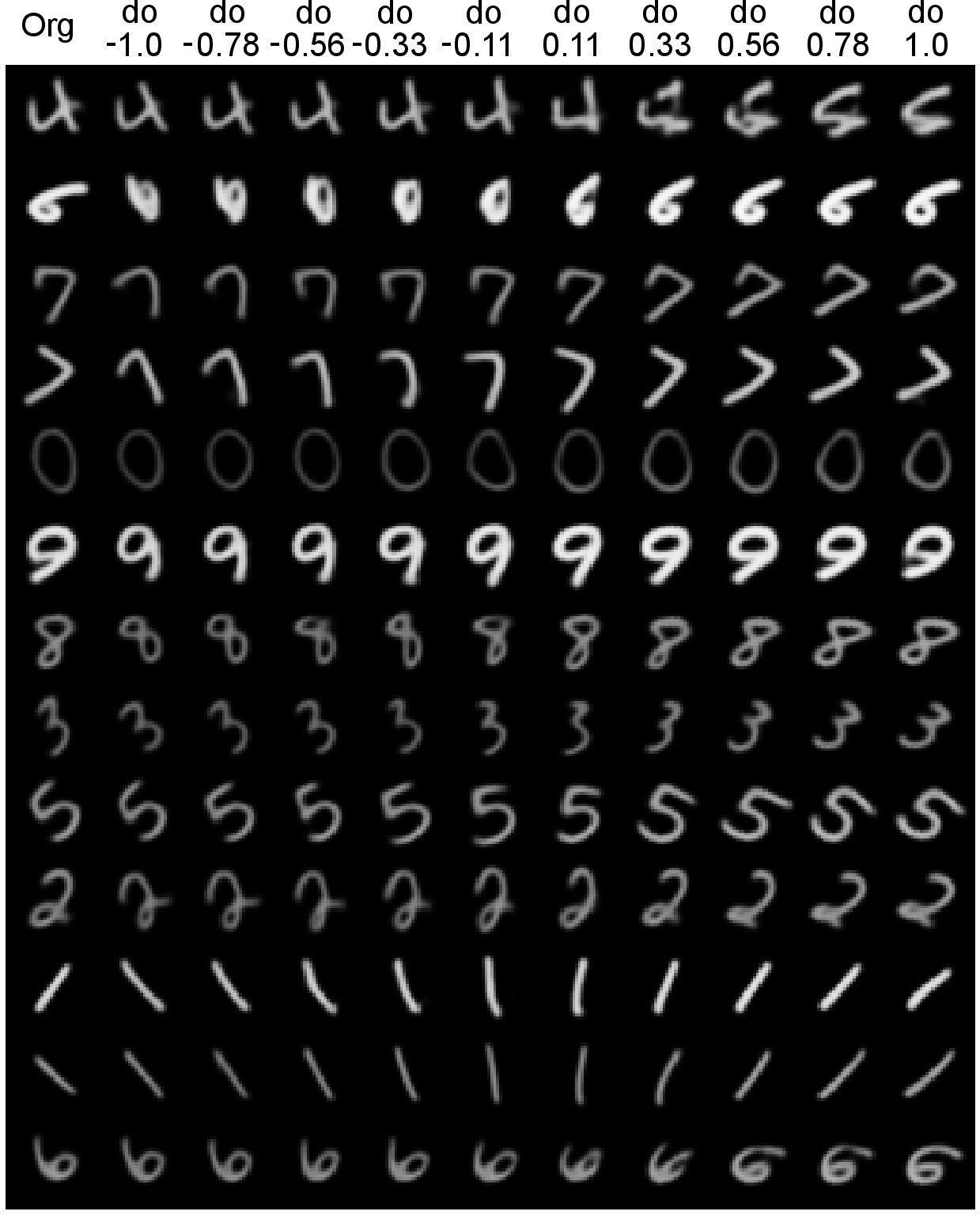}
\par\end{centering}
\caption{Counterfactual when intervening on Angle\label{fig:counterfactual-morphomnist-slant}}
\end{figure}